\documentclass[10pt,twocolumn,letterpaper]{article}
\pdfoutput=1
\usepackage{cvpr}
\usepackage{times}
\usepackage{epsfig}
\usepackage{graphicx}
\usepackage{amsmath}
\usepackage{amssymb}

\usepackage{amsthm}
\usepackage{enumerate}
\usepackage{subcaption}
\usepackage{multirow}
\usepackage{pbox}
\usepackage{array}
\usepackage{comment}

\usepackage{rotating}
\usepackage{capt-of}
\usepackage{caption}
\usepackage{algorithm}
\usepackage{algpseudocode}
\usepackage{comment}
\usepackage{longtable}
\usepackage{booktabs}       
\usepackage{xcolor,colortbl}

\definecolor{Gray}{gray}{0.85}
\definecolor{LightCyan}{rgb}{0.88,1,1}
\newcolumntype{a}{>{\columncolor{Gray}}r}
\newcolumntype{b}{>{\columncolor{white}}r}
\newcolumntype{d}{>{\columncolor{Gray}}c}

\usepackage[pagebackref=true,breaklinks=true,letterpaper=true,colorlinks,bookmarks=false,
			pdfauthor={D. Khue Le-Huu},
			pdftitle={Continuous Relaxation of MAP Inference: A Nonconvex Perspective},
			pdfsubject={Computer Vision and Pattern Recognition},
			pdfkeywords={map inference, graphical models, markov random fields, mrfs, admm, nonconvex relaxation}]{hyperref}

\cvprfinalcopy 

\newcommand\makenewtitle[1]
   {\newpage\twocolumn[{
   \null
   \vskip .375in
   \begin{center}
      {#1 \par}
      \vspace*{24pt}
      {
      \large
      \lineskip .5em
      \begin{tabular}[t]{c}
         \ifcvprfinal\Author\else Anonymous CVPR submission\\
         \vspace*{1pt}\\
Paper ID \cvprPaperID \fi
      \end{tabular}
      \par
      }
      \vskip .5em
      \vspace*{12pt}
   \end{center}}]
   }

\newtheorem{theorem}{Theorem}

\newtheorem{lemma}{Lemma}
\newtheorem*{lemma*}{Lemma}
\newtheorem{proposition}{Proposition}

\newtheorem*{subproblem*}{Subproblem}

\newtheorem*{problem*}{Problem}

\newtheorem*{conjecture*}{Conjecture}

\newtheorem{definition}{Definition}
\newtheorem*{definition*}{Definition}

\newcommand{\argmin}{\operatornamewithlimits{argmin}}
\newcommand{\argmax}{\operatornamewithlimits{argmax}}
\newcommand{\set}[1]{\left\{ #1 \right\}}
\newcommand{\abs}[1]{\left| #1 \right|}
\newcommand{\norm}[1]{\left\| #1 \right\|_2}

\newcommand{\cC}{\mathcal{C}}
\newcommand{\cE}{\mathcal{E}}

\newcommand{\cG}{\mathcal{G}}
\newcommand{\cI}{\mathcal{I}}

\newcommand{\cL}{\mathcal{L}}
\newcommand{\cM}{\mathcal{M}}
\newcommand{\cN}{\mathcal{N}}
\newcommand{\cS}{\mathcal{S}}

\newcommand{\cV}{\mathcal{V}}
\newcommand{\cX}{\mathcal{X}}

\newcommand{\RR}{\mathbb{R}}

\renewcommand{\c}{\mathbf{c}}
\renewcommand{\d}{\mathbf{d}}

\newcommand{\p}{\mathbf{p}}
\newcommand{\s}{\mathbf{s}}

\renewcommand{\r}{\mathbf{r}}
\renewcommand{\u}{\mathbf{u}}
\renewcommand{\v}{\mathbf{v}}
\newcommand{\x}{\mathbf{x}}
\newcommand{\y}{\mathbf{y}}
\newcommand{\z}{\mathbf{z}}

\newcommand{\A}{\mathbf{A}}

\newcommand{\D}{\mathbf{D}}
\newcommand{\F}{\mathbf{F}}
\newcommand{\G}{\mathbf{G}}

\renewcommand{\S}{\mathbf{S}}

\newcommand{\X}{\mathbf{X}}
\newcommand{\Y}{\mathbf{Y}}

\newcommand{\lambdab}{\mathbf{\pmb{\lambda}}}
\newcommand{\mub}{\mathbf{\pmb{\mu}}}
\newcommand{\nub}{\mathbf{\pmb{\nu}}}

\newcommand{\diag}{\mathrm{diag}}

\newcommand\inner[1]{\left\langle #1 \right\rangle}

\newcommand{\xmrf}{\x_\textsc{mrf}}
\newcommand{\xrelax}{\x_\textsc{rlx}}
\newcommand{\pw}{\mathrm{pairwise}}

\def\cvprPaperID{2695} 

\ifcvprfinal\fi
\begin{document}
\makeatletter
\title{Continuous Relaxation of MAP Inference: A Nonconvex Perspective}

\author{D. Khu\^e L\^e-Huu\textsuperscript{1, 2} \qquad Nikos Paragios\textsuperscript{1, 2, 3}\\
\normalsize\textsuperscript{1}CentraleSup\'elec, Universit\'e Paris-Saclay\qquad \textsuperscript{2}Inria\qquad\textsuperscript{3}TheraPanacea
\\
{\tt\small \{khue.le, nikos.paragios\}@centralesupelec.fr}
}\let\Author\@author
\makeatother

\maketitle

\begin{abstract} 
   In this paper, we study a nonconvex continuous relaxation of MAP inference in discrete Markov random fields (MRFs). We show that for arbitrary MRFs, this relaxation is tight, and a discrete stationary point of it can be easily reached by a simple block coordinate descent algorithm. In addition, we study the resolution of this relaxation using popular gradient methods, and further propose a more effective solution using a multilinear decomposition framework based on the alternating direction method of multipliers (ADMM). Experiments on many real-world problems demonstrate that the proposed ADMM significantly outperforms other nonconvex relaxation based methods, and compares favorably with state of the art MRF optimization algorithms in different settings.
\end{abstract}

\section{Introduction}
Finding the maximum a posteriori (MAP) configuration is a fundamental inference problem in undirected probabilistic graphical models, also known as Markov random fields (MRFs). This problem is described as follows.

Let $\s\in\cS = \cS_1\times\cdots\times\cS_n$ denote an assignment to $n$ discrete random variables $S_1,\ldots,S_n$ where each variable $S_i$ takes values in a finite set of states (or \emph{labels}) $\cS_i$. Let $\cG$ be a graph of $n$ nodes with the set of cliques $\cC$. Consider an MRF representing a joint distribution $p(\S):=p(S_1,\ldots,S_n)$ that factorizes over $\cG$, \ie $p(\cdot)$ takes the form:
\begin{equation}\label{eq:mrf-factorization}
p(\s) = \frac{1}{Z} \prod_{C\in\cC} \psi_C(s_C) \quad \forall \s\in\cS,
\end{equation}
where $s_C$ is the joint configuration of the variables in the clique $C$, $\psi_C$ are positive functions called \emph{potentials}, and $Z = \sum_{\s} \prod_{C\in\cC} \psi_C(\s_C)$ is a normalization factor called \emph{partition function}. 

The MAP inference problem consists of finding the most likely assignment to the variables, \ie:
\begin{equation}\label{eq:map-inference}
\s^* \in \argmax_{\s\in\cS} p(\s) = \argmax_{\s\in\cS} \prod_{C\in\cC} \psi_C(s_C).
\end{equation}
For each clique $C$, let $\cS_C=\prod_{i\in C}\cS_i$ be the set of its joint configurations and define
\begin{equation}
f_C(s_C) = -\log \psi_C(s_C)\quad \forall s_C\in\cS_C.
\end{equation}
It is straightforward that the MAP inference problem~\eqref{eq:map-inference} is equivalent to minimizing the following function, called the \emph{energy} of the MRF:
\begin{equation}\label{eq:energy}
e(\s) = \sum_{C\in\cC}f_C(s_C).
\end{equation}

MRF optimization has been constantly attracting a significant amount of research over the last decades. Since this problem is in general NP-hard~\cite{shimony1994finding}, various approximate methods have been proposed and can be roughly grouped into two classes: (a) methods that stay in the discrete domain, such as move-making and belief propagation~\cite{boykov2001fast, fix2011graph, komodakis2008performance, yedidia2005constructing}, or (b) methods that move into the continuous domain by solving convex relaxations such as quadratic programming (QP) relaxations~\cite{ravikumar2006quadratic} (for pairwise MRFs), semi-definite programming (SDP) relaxations~\cite{olsson2007solving}, or most prominently linear programming (LP) relaxations~\cite{globerson2008fixing, kappes2012bundle, kolmogorov2006convergent, kolmogorov2015new, komodakis2011mrf,martins2015ad3, savchynskyy2012efficient, sontag2012efficiently}.

While convex relaxations allow us to benefit from the tremendous convex optimization literature, and can be solved exactly in polynomial time, they often only produce real-valued solutions that need a further rounding step to be converted into integer ones, which can reduce significantly the accuracy if the relaxations are not tight. On the contrary, discrete methods tackle directly the original problem, but due to its combinatorial nature, this is a very challenging task. We refer to~\cite{kappes2015ijcv} for a recent comparative study of these methods on a wide variety of problems.

In this paper, we consider a different approach. We present a nonconvex continuous relaxation to the MAP inference problem for arbitrary (pairwise or higher-order) MRFs. Based on a block coordinate descent (BCD) rounding scheme that is guaranteed not to increase the energy over continuous solutions, we show that this nonconvex relaxation is tight and is actually equivalent to the original discrete problem. It should be noted that the same relaxation was previously discussed in~\cite{ravikumar2006quadratic} but only for \emph{pairwise} MRFs and, more importantly, was not directly solved. The significance of this (QP) nonconvex relaxation has remained purely theoretical since then. In this paper, we demonstrate it to be of great practical significance as well. In addition to establishing theoretical properties of this nonconvex relaxation for \emph{arbitrary} MRFs based on BCD, we study popular generic optimization methods such as projected gradient descent~\cite{bertsekas1999nonlinear} and Frank-Wolfe algorithm~\cite{frank1956algorithm} for solving it. These methods, however, are empirically shown to suffer greatly from the trivial hardness of nonconvex optimization: getting stuck in bad local minima. To overcome this difficulty, we propose a multilinear decomposition solultion based on the alternating direction method of multipliers (ADMM). Experiments on different real-world problems show that the proposed nonconvex based approach can outperform many of the previously mentioned methods in different settings.

The remainder of this paper is organized as follows. Section~\ref{sec:notation} presents necessary notation and formulation for our approach. In Section~\ref{sec:relaxation}, the nonconvex relaxation is introduced and its properties are studied, while its resolution is presented in Section~\ref{sec:solving} together with a convergence analysis in Section~\ref{sec:convergence-analysis}. Section~\ref{sec:experiments} presents experimental validation and comparison with state of the art methods. The last section concludes the paper.

\section{Notation and problem reformulation}\label{sec:notation}
It is often convenient to rewrite the MRF energy $e(\s)$~\eqref{eq:energy} using the indicator functions of labels assigned to each node. Let $\cV\subset\cC$ denote the set of nodes of the graph $\cG$. For each $i\in\cV$, let $x_i:\cS_i\to\set{0,1}$ be a function defined by $x_i(s) = 1$ if the node $i$ takes the label $s\in\cS_i$, and $x_i(s)=0$ otherwise. It is easily seen that minimizing $e(\s)$ over $\cS$ is equivalent to the following problem, where we have rewritten $e(\s)$ as a function of $\set{x_i(\cdot)}_{i\in\cV}$\footnote{In the standard LP relaxation, the product $\prod_{j\in C}x_j(s_j)$ in~\eqref{prob:mrf} is replaced with new variables $x_C(s_C)$, seen as the indicator function of the joint label assigned to the clique $C$, and the following \emph{local consistency} constraints are added: $\forall j\in C:\ \sum_{l_{C\setminus j}}x_C(s_C) = x_j(s_j)\quad \forall s_j\in\cS_j.$
}:
\begin{equation}\label{prob:mrf}
\begin{aligned} 
\mbox{min}\quad & E(\x) =  \sum_{C\in\cC} \sum_{s_C\in\cS_C}f_C(s_C)\prod_{j\in C}x_j(s_j) \\ 
\mbox{s.t.}\quad & \sum_{s\in\cS_i}x_i(s) = 1 	\quad \forall i\in\cV,\\ 
& x_i(s)\in\set{0,1}	\quad \forall s\in\cS_i, \forall i\in\cV. 
\end{aligned}
\end{equation}
For later convenience, a further reformulation using tensor notation is needed. Let us first give a brief review of tensor.

A real-valued $D\textsuperscript{th}$-order tensor $\F$ is a multidimensional array belonging to $\RR^{n_1\times n_2 \times\cdots \times n_D}$ (where $n_1,n_2,\ldots,n_D$ are positive integers). Each dimension of a tensor is called a \emph{mode}. The elements of $\F$ are denoted by $F_{i_1i_2\ldots i_D}$ where $i_d$ is the index along the mode $d$.

A tensor can be multiplied by a vector at a specific mode. Let $\mathbf{v}=(v_1,v_2,\ldots,v_{n_d})$ be an $n_d$ dimensional vector. The \emph{mode-$d$ product} of $\F$ and $\v$, denoted by $\F\bigotimes_d\v$, is a $(D-1)\textsuperscript{th}$-order tensor $\G$ of dimensions $n_1\times\cdots  \times n_{d-1}\times n_{d+1}\times\cdots \times n_D$ defined by 
\begin{equation}
G_{i_1\ldots i_{d-1}i_{d+1}\ldots i_D} = \sum_{i_d=1}^{n_d} F_{i_1\ldots i_{d}\ldots i_D}v_{i_d}\quad\forall i_{[1,D]\setminus d}.
\end{equation}
Note that the multiplication is only valid if $\v$ has the same dimension as the mode $d$ of $\F$.

The product of a tensor and multiple vectors (at multiple modes) is defined as the consecutive product of the tensor and each vector (at the corresponding mode). The order of the multiplied vectors does not matter. For example, the product of a $4\textsuperscript{th}$-order tensor $\F\in\RR^{n_1\times n_2\times n_3\times n_4}$ and two vectors $\u\in\RR^{n_2},\v\in\RR^{n_4}$ at the modes $2$ and $4$ (respectively) is an $n_1\times n_3$ tensor $\G= \F\bigotimes_2\u\bigotimes_4\v =  \F\bigotimes_4\v\bigotimes_2\u$, where 
\begin{equation}
G_{i_1i_3} = \sum_{i_2=1}^{n_2}\sum_{i_4=1}^{n_4}F_{i_1i_2i_3i_4}u_{i_2}v_{i_4}\quad\forall i_1,i_3.
\end{equation}
Let us consider for convenience the notation $\F\bigotimes_{\cI} \cM$ to denote the product of $\F$ with the set of vectors $\cM$, at the modes specified by the set of indices $\cI$ with $\abs{\cI} = \abs{\cM}$. Since the order of the vectors and the modes must agree, $\cM$ and $\cI$ are supposed to be ordered sets. By convention, $\F\bigotimes_{\cI} \cM=\F$ if $\cM = \emptyset$. Using this notation, the product in the previous example becomes 
\begin{equation}
\G = \F\bigotimes_{\set{2,4}}\set{\u,\v} = \F\bigotimes_{\set{4,2}}\set{\v,\u}.
\end{equation}

Now back to our problem~\eqref{prob:mrf}.
For any node $i$, let $\x_i = (x_i(s))_{s\in\cS_i}$ be the vector composed of all possible values of $x_i(s)$. For a clique $C=(i_1,i_2,\ldots,i_\alpha)$, the potential function $f_C(s_1,s_2,\ldots,s_\alpha)$, where $s_d\in\cS_{i_d}\forall 1\le d\le\alpha$, has $\alpha$ indices and thus can be seen as an $\alpha\textsuperscript{th}$-order tensor of dimensions $\abs{\cS_{i_1}}\times\abs{\cS_{i_2}}\times\cdots\times\abs{\cS_{i_\alpha}}$. Let $\F_C$ denote this tensor. Recall that the energy term corresponding to $C$ in~\eqref{prob:mrf} is
\begin{equation}
\sum_{s_1,s_2,\ldots,s_\alpha} f_C(s_1,s_2,\ldots,s_\alpha) x_{i_1}(s_1)x_{i_2}(s_2)\cdots x_{i_\alpha}(s_\alpha),
\end{equation}
which is clearly $\F_{C} \bigotimes_{\set{1,2,\ldots,\alpha}}\set{\x_{i_1},\x_{i_2},\ldots,\x_{i_\alpha}}$. For clarity purpose, we omit the index set and write simply $\F_{C} \bigotimes\set{\x_{i_1},\x_{i_2},\ldots,\x_{i_\alpha}}$, or equivalently $\F_{C} \bigotimes\set{\x_i}_{i\in C}$, with the assumption that each vector is multiplied at the right mode (which is the same as its position in the clique). Therefore, the energy in~\eqref{prob:mrf} becomes
\begin{equation}\label{eq:mrf-Energy}
E(\x) = \sum_{C\in\cC} \F_{C} \bigotimes\set{\x_i}_{i\in C}.
\end{equation}
Problem~\eqref{prob:mrf} can then be rewritten as
\begin{align}
\mbox{min } &\quad E(\x) \tag{\textsc{\textbf{mrf}}}\label{prob:mrf-tensor}\\ 
\mbox{s.t.} &\quad \x\in\overline{\cX} := \set{\x \ \Big| \ \mathbf{1}^\top\x_i = 1, \x_i\in\set{0,1}^{\abs{\cS_i}}	\ \forall i\in\cV}.\nonumber
\end{align}
A continuous relaxation of this problem is studied in the next section.

\section{Tight relaxation of MAP inference}\label{sec:relaxation}
By simply relaxing the constraints $\x_i\in\set{0,1}^{\abs{\cS_i}}$ in~\eqref{prob:mrf-tensor} to $\x_i\ge\mathbf{0}$, we obtain the following nonconvex relaxation:
\begin{align}
\mbox{min } &\quad E(\x) \tag{\textsc{\textbf{rlx}}}\label{prob:relax}\\ 
\mbox{s.t.} &\quad \x\in\cX := \set{\x \ \Big| \ \mathbf{1}^\top\x_i = 1, \x_i\ge\mathbf{0}	\ \forall i\in\cV}.\nonumber
\end{align}
A clear advantage of this relaxation over the LP relaxation is its compactness. Indeed, if all nodes have the same number of labels $S$, then the number of variables and number of constraints of this relaxation are respectively $\abs{\cV}S$ and $\abs{\cV}$, while for the LP relaxation these numbers are respectively $\mathcal{O}(\abs{\cC}S^D)$ and $\mathcal{O}(\abs{\cC}SD)$, with $D$ the degree of the MRF.

In this section some interesting properties of~\eqref{prob:relax} are presented. In particular, we prove that this relaxation is tight and show how to obtain a \emph{discrete} stationary point for it. Let us first propose a simple BCD algorithm to solve~\eqref{prob:relax}. Relaxation tightness and other properties follow naturally.

Let $n=\abs{\cV}$ be the number of nodes. The vector $\x$ can be seen as an $n$-block vector, where each block corresponds to each node: $\x=(\x_1,\x_2,\ldots,\x_n)$. Starting from an initial solution, BCD solves~\eqref{prob:relax} by iteratively optimizing $E$ over $\x_i$ while fixing all the other blocks. Note that our subsequent analysis is still valid for other variants of BCD, such as updating in a random order, or using subgraphs such as trees (instead of single nodes) as update blocks. To keep the presentation simple, however, we choose to update in the deterministic order $i=1,2,\ldots,n$. Each update step consists of solving
\begin{equation}\label{eq:BCD-update}
\x_i^{(k+1)} \in \argmin_{\mathbf{1}^\top\x_i = 1,\x_i\ge\mathbf{0}} E(\x_{[1,i-1]}^{(k+1)},\x_i,\x_{[i+1,n]}^{(k)}).
\end{equation}

From~\eqref{eq:mrf-Energy} it is clear that for the cliques that do not contain the node $i$, their corresponding energy terms are independent of $\x_i$. Thus, if $\cC(i)$ denotes the set of cliques containing $i$, then 
\begin{align}
	E(\x) &= \sum_{C\in\cC(i)} \F_C\bigotimes \set{\x_j}_{j\in C} + \mathrm{cst}(\x_i) \\
	&= \c_i^\top\x_i + \mathrm{cst}(\x_i),
\end{align}
where $\mathrm{cst}(\x_i)$ is a term that does not depend on $\x_i$, and
\begin{equation}\label{eq:mrf-energy-partial-derivative}
	\c_i = \sum_{C\in\cC(i)} \F_C\bigotimes \set{\x_j}_{j\in C\setminus i} \quad\forall i\in\cV.
\end{equation}
The update~\eqref{eq:BCD-update} becomes minimizing $\c_i^\top\x_i$, which can be solved using the following straightforward lemma. 
\begin{lemma}\label{lem:bcd}
Let $\c=(c_1,\ldots,c_p)\in\RR^p$, $\alpha=\argmin_{\beta}c_\beta$. The problem $\min_{\mathbf{1}^\top\u = 1,\u\ge\mathbf{0}} \c^\top\u$ has an optimal solution $\u^*=(u_1^*,\ldots,u_p^*)$ defined by $u_\alpha^* = 1$ and $u_\beta^* = 0 \ \forall \beta\neq\alpha$.
\end{lemma}
According to this lemma, we can solve~\eqref{eq:BCD-update} as follows: compute $\c_i$ using~\eqref{eq:mrf-energy-partial-derivative}, find the position $s$ of its smallest element, set $x_i(s) = 1$ and $x_i(r) = 0\ \forall r\neq s$.  Clearly, the solution $\x_i$ returned by this update step is discrete. It is easily seen that this update is equivalent to assigning the node $i$ with the following label:
\begin{equation}
s_i = \argmin_{s\in\cS_i} \sum_{C\in\cC(i)} \sum_{s_{C\setminus i}\in \cS_{C\setminus i}} f_C(s_{C\setminus i},s)\prod_{j\in C\setminus i}x_j(s_j).
\end{equation}

A sketch of the BCD algorithm is given in Algorithm~\ref{algo:BCD}.
\begin{algorithm}
\caption{Block coordinate descent for solving~\eqref{prob:relax}.}\label{algo:BCD}
\begin{algorithmic}[1]
\State Initialization: $k\gets 0$, $\x^{(0)}\in\cX$.
\State For $i=1,2,\ldots,n$: update $\x_i^{(k+1)}$ as a (discrete) solution to~\eqref{eq:BCD-update}. If $\x_i^{(k)}$ is also a discrete solution to~\eqref{eq:BCD-update}, then set $\x_i^{(k+1)} \gets \x_i^{(k)}$.
\State Let $k\gets k+1$ and go to Step 2 until $\x^{(k+1)} = \x^{(k)}$.
\end{algorithmic}
\end{algorithm} 

\noindent\textit{Remark.} Starting from a discrete solution, BCD is equivalent to Iterated Conditional Modes (ICM)~\cite{besag1986statistical}. Note however that BCD is designed for the \emph{continuous} problem~\eqref{prob:relax}, whereas ICM relies on the \emph{discrete} problem~\eqref{prob:mrf-tensor}.
\begin{proposition}\label{propos:bcd-fixed-point}
For any initial solution $\x^{(0)}$, BCD (Algorithm~\ref{algo:BCD}) converges to a discrete fixed point.
\end{proposition}
A proof is given in the supplement. We will see in Section~\ref{sec:convergence-analysis} that this fixed point is also a stationary point of~\eqref{prob:relax}.

\begin{theorem}
The continuous relaxation~\eqref{prob:relax} is tight.
\end{theorem}
\begin{proof}
Since $E(\x)$ is continuous and both $\overline{\cX}$ and $\cX$ are closed, according to the Weierstrass extreme value theorem, both~\eqref{prob:mrf-tensor} and~\eqref{prob:relax} must attain a (global) minimum, which we denote by $\xmrf$ and $\xrelax$, respectively. Obviously $E(\xrelax) \le E(\xmrf)$. Now let $\x^*$ be the solution of BCD with initialization $\x^{(0)} = \xrelax$. On the one hand, since $\mathrm{BCD}$ is a descent algorithm, we have $E(\x^*) \le E(\xrelax)$. On the other hand, since the solution returned by $\mathrm{BCD}$ is discrete, we have $\x^*\in\overline{\cX}$, yielding $E(\xmrf) \le E(\x^*)$. Putting it all together, we get $E(\x^*) \le E(\xrelax) \le E(\xmrf) \le E(\x^*)$, which implies $E(\xrelax) = E(\xmrf)$, \ie~\eqref{prob:relax} is tight.
\end{proof}
\noindent\textit{Remark.} The above proof is still valid if BCD performs only the first outer iteration. This means that one can obtain $\xmrf$ from $\xrelax$ (same energy) in polynomial time, \ie~\eqref{prob:relax} and~\eqref{prob:mrf-tensor} can be seen as equivalent. This result was previously presented in~\cite{ravikumar2006quadratic} for pairwise MRFs, here we have extended it to arbitrary order MRFs.

While BCD is guaranteed to reach a discrete stationary point of~\eqref{prob:relax}, there is no guarantee on the quality of such point. In practice, as shown later in the experiments, the performance of BCD compares poorly with state of the art MRF optimization methods. In fact, the key challenge in nonconvex optimization is that there might be many local minima, and as a consequence, algorithms can easily get trapped in \emph{bad} ones, even from multiple initializations.


In the next section, we study the resolution of~\eqref{prob:relax} using more sophisticated methods, where we come up with a multilinear decomposition ADMM that can reach very good local minima (many times even the global ones) on different real-world models.

\section{Solving the tight nonconvex relaxation}\label{sec:solving}
Since the MRF energy~\eqref{eq:mrf-Energy} is differentiable, it is worth investigating whether gradient methods can effectively optimize it. We briefly present two such methods in the next section. Then our proposed ADMM based algorithm is presented in the subsequent section. We provide a convergence analysis for all methods in Section~\ref{sec:convergence-analysis}.

\subsection{Gradient methods}\label{sec:gradient-methods}
Projected gradient descent (PGD) and Frank-Wolfe algorithm (FW) (Algorithms~\ref{algo:PGD}, \ref{algo:FW}) are among the most popular methods for solving constrained optimization. We refer to~\cite{bertsekas1999nonlinear} for an excellent presentation of these methods.

\begin{algorithm}
\caption{Projected gradient descent for solving~\eqref{prob:relax}. 
}\label{algo:PGD}
\begin{algorithmic}[1]
\State Initialization: $k\gets 0$, $\x^{(0)}\in\cX$.
\State Compute $\beta^{(k)}$ and find the projection
\begin{equation}
\s^{(k)}= \argmin_{\s\in\cX} \norm{\x^{(k)} - \beta^{(k)}\nabla E(\x^{(k)}) - \s}^2.
\end{equation}
\State Compute $\alpha^{(k)}$ and update $\x^{(k+1)} = \x^{(k)} + \alpha^{(k)}(\s^{(k)} - \x^{(k)})$. Let $k\gets k+1$ and go to Step 2.
\end{algorithmic}
\end{algorithm} 

\begin{algorithm}
\caption{Frank-Wolfe algorithm for solving~\eqref{prob:relax}.
}\label{algo:FW}
\begin{algorithmic}[1]
\State Initialization: $k\gets 0$, $\x^{(0)}\in\cX$.
\State Find $\s^{(k)}= \argmin_{\s\in\cX} \s^\top\nabla E(\x^{(k)})$.
\State Compute $\alpha^{(k)}$ and update $\x^{(k+1)} = \x^{(k)} + \alpha^{(k)}(\s^{(k)} - \x^{(k)})$. Let $k\gets k+1$ and go to Step 2.
\end{algorithmic}
\end{algorithm}

The step-sizes $\beta^{(k)}$ and $\alpha^{(k)}$ follow a chosen update rule. The most straightforward is the \emph{diminishing} rule, which has for example $\beta^{(k)} = \frac{1}{\sqrt{k+1}},\alpha^{(k)} = 1$ for PGD, and $\alpha^{(k)} = \frac{2}{k+2}$ for FW. However, in practice, these step-sizes often lead to slow convergence. A better alternative is the following \emph{line-search} ($\beta^{(k)}$ is set to $1$ for PGD):
\begin{equation}\label{eq:line-search}
\alpha^{(k)} = \argmin_{0\le \alpha\le 1}E\left(\x^{(k)} + \alpha(\s^{(k)} - \x^{(k)})\right).
\end{equation}
For our problem, this line-search can be performed efficiently because $E\left(\x^{(k)} + \alpha(\s^{(k)} - \x^{(k)})\right)$ is a polynomial of $\alpha$. Further details (including line-search, update steps, stopping conditions, as well as other implementation issues) are provided in the supplement.

\subsection{Alternating direction method of multipliers}\label{sec:admm}
Our proposed method shares some similarities with the method introduced in~\cite{lehuu2017adgm} for solving graph matching. However, to make ADMM efficient and effective for MAP inference, we add the following important practical contributions: (1) We formulate the problem using individual potential tensors at each clique (instead of a single large tensor as in~\cite{lehuu2017adgm}), which allows a better exploitation of the problem structure, as computational quantities at each node can be cached based on its neighboring nodes, yielding significant speed-ups; (2) We discuss how to choose the decomposed constraint sets that result in the best accuracy for MAP inference (note that the constraint sets for graph matching~\cite{lehuu2017adgm} are different). In addition, we present a convergence analysis for the proposed method in Section~\ref{sec:convergence-analysis}.

For the reader to quickly get the idea, let us start with an example of a second-order\footnote{Note that pairwise MRFs are also called \emph{first-order} ones.} MRF:
\begin{multline}
	E_\text{second}(\x) = \sum_{i\in \cV} \F_i \bigotimes \x_i + \sum_{ij\in\cC} \F_{ij} \bigotimes \set{\x_i,\x_j} \\ + \sum_{ijk\in\cC} \F_{ijk} \bigotimes \set{\x_i,\x_j,\x_k}.
\end{multline}
Instead of dealing directly with this high degree polynomial, which is highly challenging, the idea is to decompose $\x$ into different variables that can be handled separately using Lagrangian relaxation. To this end, consider the following multilinear function:
\begin{multline}
	F_\text{second}(\x,\y,\z) = \sum_{i\in \cV} \F_i \bigotimes \x_i + \sum_{ij\in\cC} \F_{ij} \bigotimes \set{\x_i,\y_j} \\ + \sum_{ijk\in\cC} \F_{ijk} \bigotimes \set{\x_i,\y_j,\z_k}.
\end{multline}
Clearly, $E_\text{second}(\x) = F_\text{second}(\x,\x,\x)$. Thus, minimizing $E(\x)$ is equivalent to minimizing $F_\text{second}(\x,\y,\z)$ under the constraints $\x=\y=\z$, which can be relaxed using Lagrangian based method such as ADMM.

Back to our general problem~\eqref{prob:relax}. Let $D$ denote the maximum clique size of the corresponding MRF. Using the same idea as above for decomposing $\x$ into $D$ vectors $\x^1,\x^2,\ldots,\x^D$, let us define
\begin{equation}\label{eq:energy-decomposed}
F(\x^1,\ldots,\x^D) = \sum_{d=1}^D \sum_{i_1\ldots i_d\in\cC} \F_{i_1\ldots i_d} \bigotimes \set{\x^1_{i_1},\ldots,\x^d_{i_d}}.
\end{equation}
Clearly, the energy~\eqref{eq:mrf-Energy} becomes $E(\x) = F(\x,\x,\ldots,\x)$. It is straightforward to see that~\eqref{prob:relax} is equivalent to:
\begin{equation}\label{prob:mrf-decomposed}
\begin{aligned} 
\mbox{min}\quad & F(\x^1,\x^2,\ldots,\x^D)\\ 
\mbox{s.t.}\quad & \A^1\x^1 + \cdots +\A^D\x^D = \mathbf{0},\\
&\x^d \in\cX^d,\quad d=1,\ldots,D,
\end{aligned}
\end{equation}
where $\A^1,\ldots,\A^D$ are constant matrices such that
\begin{equation}\label{eq:equality-constraint}
\A^1\x^1 + \cdots +\A^D\x^D = \mathbf{0} \Longleftrightarrow \x^1=\cdots=\x^D,
\end{equation} and $\cX^1,\ldots,\cX^D$ are closed convex sets satisfying
\begin{equation}\label{eq:set-intersection}
\cX^1\cap\cX^2\cap\cdots\cap\cX^D = \cX.
\end{equation}
Note that the linear constraint in~\eqref{prob:mrf-decomposed} is a general way to enforce $\x^1=\cdots=\x^D$ and it has an infinite number of particular instances. For example, with suitable choices of $(\A^d)_{1\le d\le D}$, this linear constraint can become either one of the following sets of constraints:
\begin{alignat}{3}
&\text{(cyclic)}&&\x^{d-1}	&&=\x^d, \quad d = 2,\ldots, D,\label{eq:cyclic}\\
&\text{(star)}&&\x^1 		&&=\x^d,\quad d = 2,\ldots, D,\label{eq:star}\\
&\text{(symmetric)}\quad&&\x^d 		&&=(\x^1+\cdots+\x^D)/D\quad \forall d.\label{eq:symmetric}
\end{alignat}
We call such an instance a \emph{decomposition}, and each decomposition will lead to a different algorithm.

The augmented Lagrangian of~\eqref{prob:mrf-decomposed} is defined by: 
\begin{multline}\label{eq:lagrangian}
	L_\rho(\x^1,\ldots,\x^D,\y) = F(\x^1,\ldots,\x^D) \\ + \y^\top\left(\sum_{d=1}^D \A^d\x^d\right) 
	+ \frac{\rho}{2}\norm{\sum_{d=1}^D \A^d\x^d}^2,
\end{multline} 
where $\y$ is the \emph{Lagrangian multiplier vector} and $\rho>0$ is called the \emph{penalty parameter}.

Standard ADMM~\cite{boyd2011distributed} solves~\eqref{prob:mrf-decomposed} by iterating:
\begin{enumerate}
	\item For $d=1,2,\ldots,D$: update $\x^{d^{(k+1)}}$ as a solution of
	\begin{equation}\label{eq:update-x}
		\min_{\x^d\in\cX^d} L_\rho(\x^{[1,d-1]^{(k+1)}},\x^d,\x^{[d+1,D]^{(k)}},\y^{(k)}).
	\end{equation}
	\item Update $\y$:
	\begin{equation}\label{eq:update-y}
		\y^{(k+1)} = \y^{(k)} + \rho\left(\sum_{d=1}^D\A^d\x^{d^{(k+1)}}\right).
	\end{equation}
\end{enumerate}
The algorithm converges if the following \emph{residual} converges to $0$ as $k\to+\infty$:
\begin{equation}\label{eq:residual}
r^{(k)} = \norm{\sum_{d=1}^D\A^d\x^{d^{(k)}}}^2 + \sum_{d=1}^D\norm{\x^{d^{(k)}} - \x^{d^{(k-1)}}}^2.
\end{equation}

We show how to solve the $\x$ update step~\eqref{eq:update-x} (the $\y$ update~\eqref{eq:update-y} is trivial). Updating $\x^d$ consists of minimizing the augmented Lagrangian~\eqref{eq:lagrangian} with respect to the $d\textsuperscript{th}$ block while fixing the other blocks.

Since $F(\x^1,\ldots,\x^D)$ is linear with respect to each block $\x^d$ (\cf~\eqref{eq:energy-decomposed}), it must have the form
\begin{equation}\label{eq:energy-factor-linear}
F(\x^{[1,d-1]},\x^d,\x^{[d+1,D]}) = \inner{\p^d,\x^d} + \mathrm{cst(\x^d)},
\end{equation}
where $\mathrm{cst(\x^d)}$ is a term that does not depend on $\x^d$. Indeed, it can be shown (detailed in the supplement) that $\p^d=(\p_1^d,\ldots,\p_n^d)$ where
\begin{multline}\label{eq:p}
\p_i^d = \sum_{\alpha=d}^D \left( \sum_{i_1\ldots i_{d-1}ii_{d+1}\ldots i_\alpha \in\cC} \F_{i_1i_2\ldots i_\alpha} \bigotimes \right. \\ \left.\set{\x^1_{i_1},\ldots,\x^{d-1}_{i_{d-1}}, \x^{d+1}_{i_{d+1}},\ldots,\x^\alpha_{i_\alpha}} \vphantom{\sum_{C=i_1\ldots i_{d-1}ii_{d+1}\ldots i_\alpha \in\cC}}\right) \forall i\in\cV.
\end{multline}
While the expression of $\p_i^d$ looks complicated, its intuition is simple: for a given node $i$ and a degree $d$, we search for all cliques satisfying two conditions: (a) their sizes are bigger than or equal to $d$, and (b) the node $i$ is at the $d\textsuperscript{th}$ position of these cliques; then for each clique, we multiply its potential tensor with all its nodes except node $i$, and sum all these products together.

Denote
\begin{equation}\label{eq:sd}
\s^d = \sum_{c=1}^{d-1}\A^c\x^c + \sum_{c=d+1}^{D}\A^c\x^c.
\end{equation}
Plugging~\eqref{eq:energy-factor-linear} and~\eqref{eq:sd} into~\eqref{eq:lagrangian} we get:
\begin{multline}\label{eq:lagrangian-quadratic}
L_\rho(\x^1,\ldots,\x^D,\y) = \frac{\rho}{2}\norm{\A^d\x^d}^2 \\ + \left(\p^d + \A^{d\top}\y + \rho\A^{d\top}\s^d\right)^\top\x^d + \mathrm{cst(\x^d)}.
\end{multline}
Therefore, the $\x$ update~\eqref{eq:update-x} becomes minimizing the quadratic function~\eqref{eq:lagrangian-quadratic} (with respect to $\x^d$) over $\cX^d$. With suitable decompositions, this problem can have a much simpler form and can be efficiently solved. For example, if we choose the \emph{cyclic} decomposition~\eqref{eq:cyclic}, then this step is reduced to finding the projection of a vector onto $\cX^d$:
\begin{equation}\label{eq:projection}
\x^{d^{(k+1)}} = \argmin_{\x^d\in\cX^d}\norm{\x^d - \c^{d^{(k)}} }^2,
\end{equation}
where $(\c_d)_{1\le d \le D}$ are defined as follows (\cf supplement):
\begin{align}
\c^{1^{(k)}} &= \x^{2^{(k)}} - \frac{1}{\rho}\left(\y^{2^{(k)}} + \p^{1^{(k)}}\right),\label{eq:cyclic-c1}\\
\c^{d^{(k)}} &= \frac{1}{2}\left(\x^{d-1^{(k+1)}} + \x^{d+1^{(k)}}\right)\label{eq:cyclic-cd}\\ &+ \frac{1}{2\rho}\left(\y^{d^{(k)}} - \y^{d+1^{(k)}} - \p^{d^{(k)}}\right),\quad 2\le d\le D-1, \notag\\
\c^{D^{(k)}} &= \x^{D-1^{(k+1)}} + \frac{1}{\rho}\left(\y^{D^{(k)}} + \p^{D^{(k)}}\right). \label{eq:cyclic-cD}
\end{align}
Here the multiplier $\y$ is the concatenation of $(D-1)$ vectors $(\y^d)_{2\le d\le D}$, corresponding to $(D-1)$ constraints in~\eqref{eq:cyclic}. 

Similar results can be obtained for other specific decompositions such as \emph{star}~\eqref{eq:star} and \emph{symmetric}~\eqref{eq:symmetric} as well. We refer to the supplement for more details. As we observed very similar performance among these decompositions, only \emph{cyclic} was included for evaluation (Section~\ref{sec:experiments}).

The ADMM procedure are sketched in Algorithm~\ref{algo:ADMM-general}. 

\begin{algorithm}
\caption{ADMM with general decomposition~\eqref{prob:mrf-decomposed} for solving~\eqref{prob:relax}.}\label{algo:ADMM-general}
\begin{algorithmic}[1]
\State Initialization: $k\gets 0$, $\y^{(0)}\gets \mathbf{0}$ and $\x^{d^{(0)}}\in\cX^d$ for $d=1,\ldots,D$.
\State For $d=1,2,\ldots,D$: update $\x^{d^{(k+1)}}$ by solving~\eqref{eq:update-x} (which is reduced to optimizing~\eqref{eq:lagrangian-quadratic} over $\cX^d$).
\State Update $\y^{(k+1)}$ using~\eqref{eq:update-y}. Let $k\gets k+1$ and go to Step 2.
\end{algorithmic}
\end{algorithm} 

In practice, we found that the penalty parameter $\rho$ and the constraint sets $(\cX^d)_{1\le d\le D}$ can greatly affect the convergence as well as the solution quality of ADMM. Let us address these together with other practical considerations.

\paragraph{Adaptive penalty}
 We observed that small $\rho$ leads to slower convergence but often better energy, and inversely for large $\rho$. To obtain a good trade-off, we follow~\cite{lehuu2017adgm} and use the following adaptive scheme: initialize $\rho_0$ at a small value and run for $I_1$ iterations (for stabilization), after that if no improvement of the residual $r^{(k)}$ is achieved every $I_2$ iterations, then we increase $\rho$ by a factor $\beta$. In addition, we stop increasing $\rho$ after it reaches some value $\rho_\mathrm{max}$, so that the convergence properties presented in the next section still apply. In the experiments, we normalize all the potentials to $[-1,1]$ and set $I_1 = 500, I_2 = 500, \beta = 1.2,\rho_0 = 0.001, \rho_\mathrm{max}=100$.

\paragraph{Constraint sets} A trivial choice of $(\cX^d)_{1\le d\le D}$ that satisfies~\eqref{eq:set-intersection} is $\cX^d = \cX\ \forall d$. Then,~\eqref{eq:projection} becomes projections onto the simplex $\set{\x_i \ | \ \mathbf{1}^\top\x_i = 1, \x_i\ge\mathbf{0}}$ for each node $i$, which can be solved using \eg the method introduced in~\cite{condat2016fast}. However, we found that this choice often produces poor quality solutions, despite converging quickly. The reason is that constraining all $\x_i^d$ to belong to a simplex will make them reach consensus faster, but without being allowed to vary more freely, they tend to bypass good solutions. The idea is to use looser constraint sets, \eg $\cX^+ := \set{\x \ | \ \x\ge \mathbf{0}}$, for which~\eqref{eq:projection} becomes simply $\x^{d^{(k+1)}} = \max(\c^{d^{(k)}},0)$. We found that leaving only one set as $\cX$ yields the best accuracy. Therefore, in our implementation we set $\cX^1 = \cX$ and $\cX^d = \cX^+\ \forall d \ge 2$.

\paragraph{Parallelization}
Since there is no dependency among the nodes in the constraint sets, the projection~\eqref{eq:projection} is clearly reduced to \emph{independent} projections at each node. Moreover, at each iteration, the expensive computation~\eqref{eq:p} of $\p_i^{d}$ can also be performed in parallel for all nodes. Therefore, the proposed ADMM is highly parallelizable. 

\paragraph{Caching}
Significant speed-ups can be achieved by avoiding re-computation of unchanged quantities. From~\eqref{eq:p} it is seen that $\p_i^d$ only depends on the decomposed variables at the neighbors of $i$. Thus, if these variables have not changed from the last iteration, then there is no need to recompute $\p_i^d$ in the current iteration. Similarly, the projection~\eqref{eq:projection} for $\x_i^d$ can be omitted if $\c_i^d$ is unchanged (\cf~\eqref{eq:cyclic-c1}--\eqref{eq:cyclic-cD}).

\section{Convergence analysis}\label{sec:convergence-analysis}
In this section, we establish some convergence results for the presented methods. Due to space constraints, proofs are provided in the supplementary material.
\begin{definition}[Stationary point]\label{def:stationary} Let $f:\RR^d\to\RR$ be a continuously differentiable function over a closed convex set $\cM$. A point $\u^*$ is called a \emph{stationary point} of the problem $\min_{\u\in\cM}f(\u)$ if and only if it satisfies
	\begin{equation}\label{eq:stationary}
	\nabla f(\u^*)^\top (\u - \u^*) \ge 0 \quad\forall \u\in\cM.
	\end{equation}
\end{definition}
Note that~\eqref{eq:stationary} is a necessary condition for a point $\u^*$ to be a local optimum (a proof can be found in~\cite{bertsekas1999nonlinear}, Chapter 2).

\begin{proposition}\label{propos:PGD-FW-stationary}
Let $\{\x^{(k)}\}$ be a sequence generated by BCD, PGD or FW (Algorithms~\ref{algo:BCD},~\ref{algo:PGD} or~\ref{algo:FW}) with line-search~\eqref{eq:line-search}. Then every limit point\footnote{A vector $\x$ is a limit point of a sequence $\{\x^{(k)}\}$ if there exists a subsequence of $\{\x^{(k)}\}$ that converges to $\x$.} of $\{\x^{(k)}\}$ is stationary.
\end{proposition}

Next, we give a convergence result for ADMM.
\begin{definition}[Karush-Kuhn-Tucker (KKT) conditions]\label{def:kkt}
A point $(\x^{*1},\x^{*2},\ldots,\x^{*D},\y^*)$ is said to be a KKT point of Problem~\eqref{prob:mrf-decomposed} if it satisfies the following KKT conditions:
\begin{align}
&\x^{*d} \in\cX^d, \qquad d=1,\ldots,D,\\
&\A^1\x^{*1} + \cdots +\A^D\x^{*D} = \mathbf{0},\label{eq:kkt-linear-constraint}\\
&\x^{*d} \in\argmin_{\x^d\in\cX^d} \set{F(\x^{*[1,d-1]},\x^d,\x^{*[d+1,D]}) + \y^{*\top}\A^d\x^d}.\label{eq:kkt-argmin)}
\end{align}
\end{definition}
Recall that by definition~\eqref{eq:equality-constraint}, condition~\eqref{eq:kkt-linear-constraint} is equivalent to $\x^{*1}=\x^{*2}=\cdots=\x^{*D}$. Therefore, any KKT point of~\eqref{prob:mrf-decomposed} must have the form $(\x^*,\x^*,\ldots,\x^*,\y^*)$ for some vector $\x^*$ and $\y^*$.

\begin{proposition}\label{propos:KKT}
Let $\{(\x^{1^{(k)}},\ldots,\x^{D^{(k)}},\y^{(k)})\}$ be a sequence generated by ADMM (Algorithm~\ref{algo:ADMM-general}). Assume that the residual $r^{(k)}$~\eqref{eq:residual} converges to $0$, then any limit point of this sequence is a KKT point of~\eqref{prob:mrf-decomposed}. 
\end{proposition}
We should note that this result is only partial, since we need the assumption that $r^{(k)}$ converges to $0$. In practice, we found that this assumption always holds if $\rho$ is large enough. Unlike gradient methods, convergence of ADMM for the kind of Problem~\eqref{prob:mrf-decomposed} (which is at the same time multi-block, non-separable and highly nonconvex) is less known and is a current active research topic. For example, global convergence of ADMM for nonconvex nonsmooth functions is established in~\cite{wang2015global}, but under numerous assumptions that are not applicable to our case.

So far for ADMM we have talked about solution to~\eqref{prob:mrf-decomposed} only and not to~\eqref{prob:relax}. In fact, we have the following result.
\begin{proposition}\label{propos:KKT-stationary}
If $(\x^*,\x^*,\ldots,\x^*,\y^*)$ is a KKT point of~\eqref{prob:mrf-decomposed} then $\x^*$ is a stationary point of~\eqref{prob:relax}.
\end{proposition}
An interesting relation of the solutions returned by the methods is the following. We say a method A can improve further a method B if we use the returned solution by B as initialization for A and A will output a better solution.

\begin{proposition}\label{propos:methods-compare}
At convergence:
\begin{enumerate}
\item BCD, PGD and FW cannot improve further each other.
\item BCD, PGD and FW cannot improve further ADMM. The inverse is not necessarily true.
\end{enumerate}
\end{proposition}
The first point follows from the fact that solutions of BCD, PGD and FW are stationary. The second point follows from Proposition~\ref{propos:KKT-stationary}. In practice, we observed that ADMM can often improve further the other methods.

\section{Experiments}\label{sec:experiments}
We compare the proposed nonconvex relaxation methods (BCD, PGD, FW and ADMM with cyclic decomposition) with the following ones (where the first four are only applicable to pairwise MRFs): $\alpha$-expansion ($\alpha$-Exp)~\cite{boykov2001fast}, fast primal-dual (FastPD)~\cite{komodakis2008performance}, convex QP relaxation (CQP)~\cite{ravikumar2006quadratic}, sequential tree reweighted message passing (TRWS)~\cite{kolmogorov2006convergent}, tree reweighted belief propagation (TRBP)~\cite{wainwright2005map}, alternating direction dual decomposition (ADDD)~\cite{martins2015ad3}, bundle dual decomposition\footnote{We also included subgradient dual decomposition~\cite{komodakis2011mrf} but found that its performance was generally worse than bundle dual decomposition, thus we excluded it from the presented results.}~(BUNDLE)~\cite{kappes2012bundle}, max-product linear programming (MPLP)~\cite{globerson2008fixing} and its extension (MPLP-C)~\cite{sontag2012efficiently}, extension of $\alpha$-expansion to higher-order using reduction technique ($\alpha$-Fusion)~\cite{fix2011graph}, generalization of TRWS to higher-order (SRMP)~\cite{kolmogorov2015new}. The code of most methods are obtained via either the OpenGM library~\cite{opengm-library} or from the authors' websites, except for CQP~\cite{ravikumar2006quadratic} we use our implementation as no code is publicly available (\cf supplement for implementation details).

For BCD, PGD and FW, we run for $5$ different initializations (solution of the unary potentials plus $4$ other completely random) and pick the best one. For ADMM, we use a \emph{single} homogeneous initial solution: $x_i(s) = \frac{1}{\abs{\cS_i}}\forall s\in \cS_i$ (we find that ADMM is quite insensitive to initialization). For these methods, BCD is used as a final rounding step.\footnote{BCD cannot improve further the solution according to Proposition~\ref{propos:methods-compare}.}

\begin{table}[!htb]
\small
\centering
\caption{List of models used for evaluation.}
\label{tab:models}
\setlength{\tabcolsep}{0.8mm}
\begin{tabular}{lcccccc}
	\toprule
	Model                        & No.\textsuperscript{$*$} & $\abs{\cV}$\textsuperscript{$**$}  & $S$\textsuperscript{$\dagger$} & $D$\textsuperscript{$\ddagger$} &  Structure  &      Function       \\ \midrule
	Inpainting                   &  4  &    14400     &      4      &  2  & grid-N4/N8  &          Potts           \\
	Matching                     &  4  &   $\sim$20   &  $\sim$20   &  2  & full/sparse &         general          \\
	1\textsuperscript{st} stereo &  3  & $\sim$100000 &    16-60    &  2  &   grid-N4   & 		TL/TS \\
	Segmentation                 & 10  &     1024     &      4      &  4  &   grid-N4   &         g-Potts          \\
	2\textsuperscript{nd} stereo &  4  & $\sim$25000  &     14      &  3  &   grid-N4   &         general          \\ \bottomrule
\multicolumn{7}{c}{\textsuperscript{$*$},\textsuperscript{$**$},\textsuperscript{$\dagger$},\textsuperscript{$\ddagger$}: number of instances, variables, labels, and MRF degree}
\end{tabular}
\end{table}

The methods are evaluated on several real-world vision tasks: image inpainting, feature matching, image segmentation and stereo reconstruction. All methods are included whenever applicable. A summary of the models are given in Table~\ref{tab:models}. Except for higher-order stereo, these models were previously considered in a recent benchmark for evaluating MRF optimization methods~\cite{kappes2015ijcv}, and their model files are publicly available\footnote{\url{http://hciweb2.iwr.uni-heidelberg.de/opengm/index.php?l0=benchmark}}. For higher-order stereo, we use the model presented in~\cite{woodford2009global}, where the disparity map is encouraged to be piecewise smooth using a second-order prior, and the labels are obtained from $14$ pre-generated piecewise-planar proposals. We apply this model to $4$ image pairs (\emph{art, cones, teddy, venus}) of the Middlebury dataset~\cite{scharstein2003high} (at half resolution, due to the high inference time). We refer to~\cite{kappes2015ijcv} and to the supplement for further details on all models.


\begin{table*}[t!]
\scriptsize
\centering
\caption{Results on pairwise models.}
\label{tab:pairwise}
\setlength{\tabcolsep}{1.5mm}
\begin{tabular}{lrarrarrarrar}
	\toprule
				& \multicolumn{3}{c}{Inpainting N4 (2 instances)} & \multicolumn{3}{c}{Inpainting N8 (2 instances)} & \multicolumn{3}{c}{Feature matching (4 instances)} & \multicolumn{3}{c}{Pairwise stereo (3 instances)}\\ \toprule
	algorithm   &     time (s) &   value &   bound & time (s) 	 &   value &   bound  &     time (s) &     value &   bound &     time (s) &     value &   bound   \\ \midrule
	$\alpha$-Exp& $   0.02$ & $ \mathbf{454.35}$ & $-\infty$ & $   0.78$ & $ 465.02$ & $-\infty$ & $  -^*$ &   $  -^*$ & $-^*$ &   $  14.75$ & $   1617196.00$ &       $-\infty$  \\
	FastPD      & $   0.03$ & $ 454.75$ & $ 294.89$ & $   0.15$ & $ 465.02$ & $ 136.28$ & $  -^*$ &   $  -^*$ & $-^*$ &   $   7.14$ & $   1614255.00$ & $301059.33$   \\ \cmidrule{1-1} 
    TRBP        & $  23.45$ & $ 480.27$ & $-\infty$ & $  64.00$ & $ 495.80$ & $-\infty$ & $   0.00$ & $1.05\times 10^{11}$ & $-\infty$ &  $2544.12$ & $   1664504.33$ &       $-\infty$ \\\cmidrule{1-1} 
	ADDD 		& $  15.87$ & $ 483.41$ & $ 443.71$ & $  35.78$ & $ 605.14$ & $ 450.95$ & $   3.16$ & $1.05\times 10^{11}$ & $  16.35$ &  $-^{**}$ &         $-^{**}$  		&         $-^{**}$\\
	MPLP        & $  55.32$ & $ 497.16$ & $ 411.94$ & $ 844.97$ & $ 468.97$ & $ 453.55$ & $   0.47$ &  $0.65\times 10^{11}$ & $  15.16$ &  $-^{**}$ &         $-^{**}$ &         $-^{**}$\\
	MPLP-C      & $1867.20$ & $ 468.88$ & $ 448.03$ & $2272.39$ & $ 479.54$ & $ 454.35$ & $   6.04$ &   $  \mathbf{21.22}$ & $  21.22$ &  $-^{**}$ &         $-^{**}$ &         $-^{**}$\\
	BUNDLE      & $  36.18$ & $ 455.25$ & $ 448.23$ & $ 111.74$ & $ 465.26$ & $ 455.43$ & $   2.33$ &  $0.10\times 10^{11}$ & $  14.47$ &  $2039.47$ & $   1664707.67$ & $   1583742.13$\\
	TRWS        & $   1.37$ & $ 490.48$ & $ 448.09$ & $  16.23$ & $ 500.09$ & $ 453.96$ & $   0.05$ &   $  64.19$ & $  15.22$ &  $ 421.20$ & $   \mathbf{1587961.67}$ & $   1584746.58$\\ \cmidrule{1-1} 
    CQP 		& $   1.92$ & $1399.51$ & $-\infty$ & $  11.62$ & $1178.91$ & $-\infty$ & $   0.08$ &   $ 127.01$ & $-\infty$ &  $3602.01$ & $  11408446.00$ &       $-\infty$\\ \cmidrule{1-1} 
	BCD         & $   0.11$ & $ 485.88$ & $-\infty$ & $   0.29$ & $ 481.95$ & $-\infty$ & $   0.00$ &   $  84.86$ & $-\infty$ &  $  10.82$ & $   7022189.00$ &       $-\infty$\\
	FW          & $   1.10$ & $ 488.23$ & $-\infty$ & $   5.94$ & $ 489.82$ & $-\infty$ & $  20.10$ &   $  66.71$ & $-\infty$ &  $1989.12$ & $   6162418.00$ &       $-\infty$\\
	PGD         & $   0.81$ & $ 489.80$ & $-\infty$ & $   5.19$ & $ 489.82$ & $-\infty$ & $  13.21$ &   $  58.52$ & $-\infty$ &  $1509.49$ & $   5209092.33$ &       $-\infty$\\ 
	ADMM 		& $   9.84$ & $ \mathbf{454.35}$ & $-\infty$ & $  40.64$ & $ \mathbf{464.76}$ & $-\infty$ & $   0.31$ &   $  75.12$ & $-\infty$ &  $2377.66$ & $   1624106.00$ &       $-\infty$\\\bottomrule
	\multicolumn{13}{c}{$^*$Method not applicable $^{**}$Prohibitive execution time (time limit not working) or prohibitive memory consumption}
\end{tabular}
\end{table*}

The experiments were carried out on a $64$-bit Linux machine with a $3.4$GHz processor and $32$GB of memory. A time limit of $1$ hour was set for all methods. In Tables~\ref{tab:pairwise} and~\ref{tab:high-order}, we report the runtime\footnote{For a fair comparison, we used the \emph{single-thread} version of ADMM.}, the energy value of the final integer solution as well as the lower bound if available, averaged over all instances of a particular model. The detailed results are given in the supplement.

In general, ADMM significantly outperforms BCD, PGD, FW and is the only nonconvex relaxation method that compares favorably with the other methods. In particular, it outperforms TRBP, ADDD, BUNDLE, MPLP, MPLP-C and CQP on all models (except MPLP-C on matching), and outperforms FastPD, $\alpha$-Exp/$\alpha$-Fusion and TRWS on small or medium sized models (\ie other than stereo).

On image inpainting (Table~\ref{tab:pairwise}), ADMM produces the lowest energies on all instances, while being relatively fast. Surprisingly TRWS performs poorly on these models, even worse than BCD, PGD and FW.

The feature matching model (Table~\ref{tab:pairwise}) is a typical example showing that the standard LP relaxation can be very loose. All methods solving its dual produce very poor results (despite reaching relatively good lower bounds). They are largely outperformed by TRWS and nonconvex relaxation methods (BCD, PGD, FW, ADMM). On this problem, MPLP-C reaches the global optimum for all instances.

\begin{table}[H]
\scriptsize
\centering
\caption{Results on higher-order models.}
\label{tab:high-order}
\setlength{\tabcolsep}{1.2mm}
\begin{tabular}{lrarrar}
	\toprule
					& \multicolumn{3}{c}{Segmentation (10 instances)} & \multicolumn{3}{c}{Second-order stereo (4 instances)}\\ \toprule
	algorithm       &  time (s) &              value &     bound &  time (s) &              value &     bound \\ \midrule
	$\alpha$-Fusion & $   0.05$ &          $1587.13$ & $-\infty$ & $  50.03$ & $ 14035.91$ & $-\infty$   \\ \cmidrule{1-1} 
    TRBP         	& $  18.20$ &          $1900.84$ & $-\infty$ & $3675.90$ & $ 14087.40$ & $-\infty$   \\ \cmidrule{1-1} 
	ADDD	 		& $   6.36$ &          $3400.81$ & $1400.33$ & $4474.83$ & $ 14226.93$ & $ 13752.73$ \\ 
	MPLP            & $   9.68$ &          $4000.44$ & $1400.30$ & $  -^*$	 & $        -^*$ & $-^*$ 		 \\ 
	MPLP-C          & $3496.50$ &          $4000.41$ & $1400.35$ & $  -^*$	 & $        -^*$ & $-^*$ 		 \\ 
	BUNDLE          & $ 101.56$ &          $4007.73$ & $1392.01$ & $3813.84$ & $ 15221.19$ & $ 13321.96$ \\ 
	SRMP            & $   0.13$ & $\mathbf{1400.57}$ & $1400.57$ & $3603.41$ & $ \mathbf{13914.82}$ & $ 13900.87$ \\ \cmidrule{1-1} 
	BCD             & $   0.14$ &    $ 12518.59$ & $-\infty$ & $  59.59$ & $ 14397.22$ & $-\infty$   \\ 
	FW              & $  21.23$ &          $5805.17$ & $-\infty$ & $1749.19$ & $ 14272.54$ & $-\infty$   \\ 
	PGD             & $  51.04$ &          $5513.02$ & $-\infty$ & $3664.92$ & $ 14543.65$ & $-\infty$   \\
 	ADMM 			& $  97.37$ &          $1400.68$ & $-\infty$ & $3662.13$ & $ 14068.53$ & $-\infty$   \\  \bottomrule
 	\multicolumn{7}{c}{$^*$Prohibitive execution time (time limit not working)}
\end{tabular}
\end{table}

On image segmentation (Table~\ref{tab:high-order}), SRMP performs exceptionally well, producing the global optimum for all instances while being very fast. ADMM is only slightly outperformed by SRMP in terms of energy value, while both clearly outperform the other methods.

On large scale models such as stereo, TRWS/SRMP perform best in terms of energy value, followed by move making algorithms (FastPD, $\alpha$-Exp/$\alpha$-Fusion) and ADMM. An example of estimated disparity maps is given in Figure~\ref{fig:stereo} for SRMP and nonconvex relaxation methods. Results for all methods are given in the supplement.

An interesting observation is that CQP performs worse than nonconvex methods on all models (and worst overall), which means simply solving the QP relaxation in a straightforward manner is already better than adding a sophisticated convexification step. This finding is for us rather surprising.

\begin{figure}[!htb]
 \centering
     \begin{subfigure}[t]{0.30\linewidth}
     \centering
         \includegraphics[width=\linewidth]{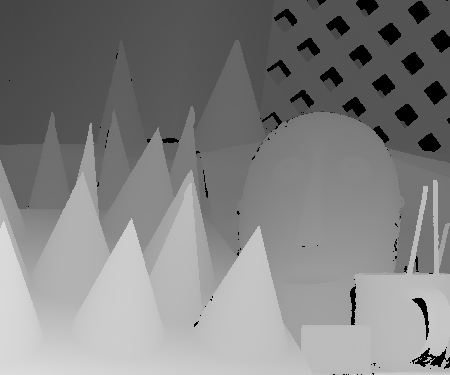}\vspace{-5pt}
         \caption{\scriptsize Ground-truth}
     \end{subfigure}%
     ~
     \begin{subfigure}[t]{0.30\linewidth}
     \centering
             \includegraphics[width=\linewidth]{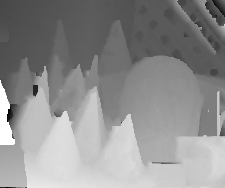}\vspace{-5pt}
             \caption{\scriptsize SRMP (\textbf{18433.01})}
         \end{subfigure}%
     ~
     \begin{subfigure}[t]{0.30\linewidth}
     \centering
             \includegraphics[width=\linewidth]{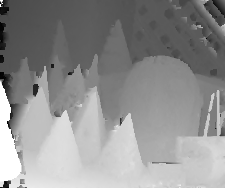}\vspace{-5pt}
             \caption{\scriptsize BCD (18926.70)}
         \end{subfigure}\\
     \begin{subfigure}[t]{0.30\linewidth}
     \centering
             \includegraphics[width=\linewidth]{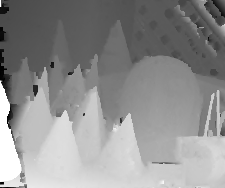}\vspace{-5pt}
             \caption{\scriptsize FW (18776.26)}
         \end{subfigure}%
     ~
     \begin{subfigure}[t]{0.30\linewidth}
    \centering
            \includegraphics[width=\linewidth]{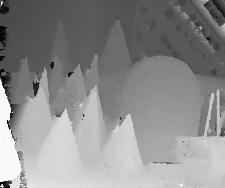}\vspace{-5pt}
            \caption{\scriptsize PGD (19060.17)}
        \end{subfigure}%
     ~
     \begin{subfigure}[t]{0.30\linewidth}
     \centering
             \includegraphics[width=\linewidth]{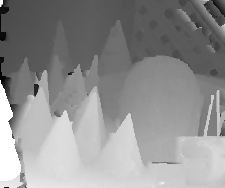}\vspace{-5pt}
             \caption{\scriptsize ADMM (18590.87)}
         \end{subfigure}
     \caption{\label{fig:stereo}Estimated disparity maps and energy values on higher-order stereo model.}
 \end{figure}
 
 \section{Conclusion}\label{sec:conclusion}
 We have presented a tight nonconvex relaxation for the problem of MAP inference and studied four different methods for solving it: block coordinate descent, projected gradient descent, Frank-Wolfe algorithm, and ADMM. Due to the high nonconvexity, it is very challenging to obtain good solutions to this relaxation, as shown by the performance of the first three methods. The latter, however, outperforms many existing methods and thus demonstrates that directly solving the nonconvex relaxation can lead to very accurate results. These methods are memory efficient, thanks to the small number of variables and constraints (as discussed in Section~\ref{sec:relaxation}). On top of that, the proposed ADMM algorithm is also highly parallelizable (as discussed in Section~\ref{sec:admm}), which is not the case for methods like TRWS or SRMP. Therefore, ADMM is also suitable for distributed or real-time applications on GPUs.

\paragraph{Acknowledgements} This research was partially supported by the European Research Council Starting Grant DIOCLES (ERC-STG-259112) and the PUF 4D Vision project (Partner University Fund). The authors thank Jean-Christophe Pesquet for useful discussion on gradient-based methods, and thank the anonymous reviewers for their insightful comments.
{\small
\bibliographystyle{ieee}
\bibliography{ref}
}


\makenewtitle{\bf\Large Continuous Relaxation of MAP Inference: A Nonconvex Perspective\\ {\large Supplementary Material}}

\begin{abstract} 
   We give proofs of the presented theoretical results in Appendix~\ref{append:proofs}, implementation details of the methods in Appendix~\ref{append:methods}, and experiment details in Appendix~\ref{append:experiments}.
\end{abstract}  

\renewcommand\thesection{\Alph{section}}
\renewcommand\thesubsection{\thesection.\arabic{subsection}}
\setcounter{section}{0}%
\section{Proofs}\label{append:proofs}
\subsection{Proof of Proposition~\ref{propos:bcd-fixed-point}}
Clearly, BCD stops when there is no \emph{strict} descent of the energy. Since the solution at each iteration is discrete and the number of nodes as well as the number of labels are finite, BCD must stop after a finite number of iterations. Suppose that this number is $k$: $E(\x^{(k+1)}) = E(\x^{(k)})$. At each inner iteration (\ie Step 2 in Algorithm~\ref{algo:BCD}), the label of a node is changed to a new label only if the new label can produce \emph{strictly} lower energy. Therefore, the labeling of $\x^{(k+1)}$ and $\x^{(k)}$ must be the same because they have the same energy, which implies $\x^{(k+1)} = \x^{(k)}$, \ie
$\x^{(k)}$ is a fixed point. 

\subsection{Proof of Equation~(\ref{eq:p})}

Recall from~\eqref{eq:energy-decomposed} that
\begin{equation}\label{eq:energy-decomposed-alpha}
F(\x^1,\ldots,\x^D) = \sum_{\alpha=1}^D \sum_{i_1\ldots i_\alpha\in\cC} \F_{i_1\ldots i_\alpha} \bigotimes \set{\x^1_{i_1},\ldots,\x^\alpha_{i_\alpha}},
\end{equation}
Clearly, the terms corresponding to any $\alpha < d$ do not involve $\x^d$. Thus, we can rewrite the above as
\begin{multline}
F(\x^1,\ldots,\x^D) = \mathrm{cst(\x^d)} \\ +  \sum_{\alpha=d}^D \sum_{i_1\ldots i_\alpha\in\cC} \F_{i_1\ldots i_\alpha} \bigotimes \set{\x^1_{i_1},\ldots,\x^\alpha_{i_\alpha}}.\label{eq:p-double-sum}
\end{multline}
We will show that the last double sum can be written as $\sum_{i\in\cV}\inner{\p_i^d,\x_i^d}$, where $\p_i^d$ is given by~\eqref{eq:p}. The idea is to regroup, for each node $i$, all terms that contain $\x_i$. Indeed, for a given $d$ we have the identity:
\begin{equation}
\sum_{i_1i_2\ldots i_\alpha\in\cC} =  \sum_{i_d\in\cV} \sum_{i_1\ldots i_{d-1}i_di_{d+1}\ldots i_\alpha \in\cC}.
\end{equation}
Therefore, the double sum in~\eqref{eq:p-double-sum} becomes
\begin{equation}
\sum_{\alpha=d}^D\sum_{i_d\in\cV} \sum_{i_1\ldots i_{d-1}i_di_{d+1}\ldots i_\alpha \in\cC} \F_{i_1\ldots i_\alpha} \bigotimes \set{\x^1_{i_1},\ldots,\x^\alpha_{i_\alpha}}.
\end{equation}
Rearranging the first and second sums we obtain
\begin{equation}
\sum_{i_d\in\cV} \sum_{\alpha=d}^D \sum_{i_1\ldots i_{d-1}i_di_{d+1}\ldots i_\alpha \in\cC} \F_{i_1\ldots i_\alpha} \bigotimes \set{\x^1_{i_1},\ldots,\x^\alpha_{i_\alpha}}.
\end{equation}
With the change of variable $i\gets i_d$ this becomes
\begin{multline}
\sum_{i\in\cV} \sum_{\alpha=d}^D \sum_{i_1\ldots i_{d-1}ii_{d+1}\ldots i_\alpha \in\cC} \F_{i_1\ldots i_\alpha} \bigotimes \set{\x^1_{i_1},\ldots,\x^\alpha_{i_\alpha}}.
\end{multline}
Now by factoring out $\x_i^d$ for each $i\in\cV$ the above becomes

\begin{multline}
\sum_{i\in\cV} \left(\sum_{\alpha=d}^D \sum_{i_1\ldots i_{d-1}ii_{d+1}\ldots i_\alpha \in\cC} \F_{i_1i_2\ldots i_\alpha} \bigotimes \right. \\ \left. \set{\x^1_{i_1},\ldots,\x^{d-1}_{i_{d-1}}, \x^{d+1}_{i_{d+1}},\ldots,\x^\alpha_{i_\alpha}} \vphantom{\sum_{i_1\ldots i_{d-1}ii_{d+1}\ldots i_\alpha \in\cC}}\right)^\top \x_i^d,
\end{multline}
which is $\sum_{i\in\cV}\inner{\p_i^d,\x_i^d}$, where $\p_i^d$ is given by~\eqref{eq:p}, QED.

\subsection{Proof of Equations~(\ref{eq:cyclic-c1})--(\ref{eq:cyclic-cD})}
See Appendix~\ref{append:admm}, page~\pageref{append:admm} on the details of ADMM.

\subsection{Proof of Proposition~\ref{propos:PGD-FW-stationary}}
For PGD and FW, the result holds for general continuously differentiable function $E(\cdot)$ and closed convex set $\cX$. We refer to~\cite{bertsekas1999nonlinear} (Sections 2.2.2 and 2.3.2) for a proof. Below we give a proof for BCD.

In Proposition~\ref{propos:bcd-fixed-point} we have shown that BCD reaches a discrete fixed point $\x^{(k)}$ after a finite number of iterations $k$. Now, we show that this fixed point is stationary. Define $\Delta_i = \set{\u\in\RR^{\abs{\cS_i}}: \u\ge \mathbf{0}, \mathbf{1}^\top \u = 1}\ \forall i\in\cV$ and let $\x^* = \x^{(k+1)}=\x^{(k)}$. At the last $i\textsuperscript{th}$ inner iteration~\eqref{eq:BCD-update} we have:
\begin{equation}
E(\x_{[1,i-1]}^{(k+1)},\x_i,\x_{[i+1,n]}^{(k)}) \ge E(\x_{[1,i-1]}^{(k+1)},\x_i^{(k+1)},\x_{[i+1,n]}^{(k)})
\end{equation}
for all $\x_i\in\Delta_i$, which is
\begin{equation}
E(\x_{[1,i-1]}^*,\x_i,\x_{[i+1,n]}^*) \ge E(\x_{[1,i-1]}^*,\x_i^*,\x_{[i+1,n]}^*)
\end{equation}
for all $\x_i\in\Delta_i$. Define for each $i$ the function
\begin{equation}
E_i^*(\x_i) = E(\x_1^*,\ldots,\x_{i-1}^*,\x_i,\x_{i+1}^*,\ldots,\x_n^*).
\end{equation}
Obviously $E_i^*(\x_i)$ is continuously differentiable as it is linear. Since $\x_i^*$ is a minimizer of $E_i^*(\x_i)$ over $\Delta_i$, which is closed and convex, according to~\eqref{eq:stationary} (which is a necessary optimality condition) we have $\nabla E_i^*(\x_i^*)^\top(\x_i - \x_i^*)\ge 0\ \forall \x_i\in\Delta_i$. Notice that
\begin{equation}
\nabla E(\x^*) = \begin{bmatrix}
\frac{\partial E(\x^*)}{\partial \x_1} \\
\vdots \\
\frac{\partial E(\x^*)}{\partial \x_n}
\end{bmatrix} =
\begin{bmatrix}
\nabla E_1^*(\x_1^*) \\
\vdots \\
\nabla E_n^*(\x_n^*)
\end{bmatrix},
\end{equation}
we have
\begin{equation}
\nabla E(\x^*)^\top(\x - \x^*) =  \sum_{i=1}^n   \nabla E_i^*(\x_i^*)^\top(\x_i - \x_i^*).
\end{equation}
Since each term in the last sum is non-negative, we have $\nabla E(\x^*)^\top(\x - \x^*) \ge 0\ \forall\x\in\cX$, \ie $\x^*$ is stationary.

\subsection{Proof of Proposition~\ref{propos:KKT}}
By Definition~\ref{def:kkt}, a point $(\x^1,\ldots,\x^D,\y)$ is a KKT of~\eqref{prob:mrf-decomposed} if and only if it has the form $(\x^*,\ldots,\x^*,\y^*)$ (where $\x^*\in\cX$) and at the same time satisfies
\begin{equation}\label{eq:kkt-appendix}
\x^{*d} \in\argmin_{\x^d\in\cX^d} \set{F(\x^*,\ldots,\x^*,\x^d,\x^*,\ldots,\x^*) + \y^{*\top}\A^d\x^d}
\end{equation}
for all $d$, which is \emph{equivalent} to
\begin{multline}\label{eq:kkt-appendix-inequality}
\left(\frac{\partial F}{\partial \x^d}(\x^*,\ldots,\x^*) + \A^{d\top}\y^*\right)^\top(\x^d-\x^*) \ge 0\\ \forall\x^d\in\cX^d,\forall d.
\end{multline}
The equivalence (``$\Leftrightarrow$'') follows from the fact that the objective function (with respect to $\x^d$) in~\eqref{eq:kkt-appendix} is convex. This is a well-known result in convex analysis, which we refer to Bertsekas, Dimitri P., Angelia Nedi, and Asuman E. Ozdaglar. \emph{Convex analysis and optimization.}" (2003) (Proposition 4.7.2) for a proof. Note that from the necessary optimality condition~\eqref{eq:stationary} we can only have the ``$\Rightarrow$'' direction.

We need to prove that the sequence $\{(\x^{1^{(k)}},\ldots,\x^{D^{(k)}},\y^{(k)})\}$ generated by ADMM satisfies the above conditions (under the assumption that the residual $r^{(k)}$ converges to $0$).

Let $(\x^{*1},\x^{*2},\ldots,\x^{*D},\y^*)$ be a limit point of $\{(\x^{1^{(k)}},\ldots,\x^{D^{(k)}},\y^{(k)})\}$ (thus $\x^{*d}\in\cX^d\ \forall d$ since $(\cX^d)_{1\le d\le D}$ are closed), and define a subsequence that converges to this limit point by $\{(\x^{1^{(l)}},\ldots,\x^{D^{(l)}},\y^{(l)})\}$, $l\in\cL \subset \mathbb{N}$ where $\cL$ denotes the set of indices of this subsequence. We have
\begin{equation}\label{eq:lim-subsequence}
\lim\limits_{\substack{l\to +\infty \\ l\in\cL}} (\x^{1^{(l)}},\ldots,\x^{D^{(l)}},\y^{(l)}) = (\x^{*1},\x^{*2},\ldots,\x^{*D},\y^*).
\end{equation}
Since the residual $r^{(k)}$~\eqref{eq:residual} converges to $0$, we have
\begin{align}
\lim\limits_{\substack{l\to +\infty \\ l\in\cL}}\left(\sum_{d=1}^D\A^d\x^{d^{(l)}}\right) &= \mathbf{0},\label{eq:lim-primal-residual}\\
\lim\limits_{\substack{l\to +\infty \\ l\in\cL}}\left(\x^{d^{(l+1)}} - \x^{d^{(l)}}\right) &= \mathbf{0} \quad \forall d.\label{eq:lim-dual-residual}
\end{align}
On the one hand, combining~\eqref{eq:lim-subsequence} and~\eqref{eq:lim-dual-residual} we get
\begin{multline}\label{eq:lim-subsequence-2}
\lim\limits_{\substack{l\to +\infty \\ l\in\cL}} (\x^{1^{(l+1)}},\ldots,\x^{D^{(l+1)}},\y^{(l+1)}) \\ = (\x^{*1},\x^{*2},\ldots,\x^{*D},\y^*).
\end{multline}
(Note that the above is different from~\eqref{eq:lim-subsequence} because $l+1$ might not belong to $\cL$.) On the other hand, combining~\eqref{eq:lim-subsequence} and~\eqref{eq:lim-primal-residual} we get
\begin{equation}
\sum_{d=1}^D\A^d\x^{*d} = \mathbf{0},
\end{equation}
which is, according to~\eqref{eq:equality-constraint}, equivalent to
\begin{equation}
\x^{*1} = \x^{*2} = \cdots = \x^{*D}.
\end{equation}
Let $\x^*\in\cX$ denote the value of these vectors. From~\eqref{eq:lim-subsequence} and~\eqref{eq:lim-subsequence-2} we have
\begin{alignat}{3}
\lim\limits_{\substack{l\to +\infty \\ l\in\cL}} \x^{d^{(l)}} &= \lim\limits_{\substack{l\to +\infty \\ l\in\cL}} \x^{d^{(l+1)}} &&= \x^* \quad\forall d, \label{eq:lim-x} \\
\lim\limits_{\substack{l\to +\infty \\ l\in\cL}} \y^{(l)} &= \lim\limits_{\substack{l\to +\infty \\ l\in\cL}} \y^{(l+1)} &&= \y^*. \label{eq:lim-y}
\end{alignat}

It only remains to prove that $(\x^*,\ldots,\x^*,\y^*)$ satisfies~\eqref{eq:kkt-appendix-inequality}. Let us denote for convenience
\begin{equation}
\z_d^{(k)} = (\x^{[1,d]^{(k+1)}},\x^{[d+1,D]^{(k)}}) \quad\forall d.
\end{equation} 
According to~\eqref{eq:stationary}, the $\x$ update~\eqref{eq:update-x} implies
\begin{multline}\label{eq:dL-optimality}
\left(\frac{\partial L_\rho}{\partial \x^d} (\z_d^{(k)},\y^{(k)})\right)^\top \left(\x^d - \x^{d^{(k+1)}}\right) \\ \ge 0\quad\forall \x^d\in\cX^d, \forall d, \forall k.
\end{multline}
Since $L_\rho$~\eqref{eq:lagrangian} is continuously differentiable, applying~\eqref{eq:lim-x} and~\eqref{eq:lim-y} we obtain
\begin{multline}\label{eq:dL-lim}
\lim\limits_{\substack{l\to +\infty \\ l\in\cL}} \frac{\partial L_\rho}{\partial \x^d}(\z_d^{(l)},\y^{(l)}) = \frac{\partial L_\rho}{\partial \x^d}(\x^*,\ldots,\x^*,\y^*)\quad\forall d.
\end{multline}
Let $k=l$ in~\eqref{eq:dL-optimality} and take the limit of that inequality, taking into account~\eqref{eq:lim-x} and~\eqref{eq:dL-lim}, we get
\begin{equation}
\left(\frac{\partial L_\rho}{\partial \x^d}(\x^*,\ldots,\x^*,\y^*)\right)^\top(\x^d - \x^*) \ge 0\quad \forall \x^d\in\cX^d,\forall d.
\end{equation} 
From the definition of $L_\rho$~\eqref{eq:lagrangian} we have
\begin{align}
&\frac{\partial L_\rho}{\partial \x^d}(\x^*,\ldots,\x^*,\y^*) \notag \\
= &\frac{\partial F}{\partial \x^d}(\x^*,\ldots,\x^*) + \A^{d\top}\y^* + \rho\A^{d\top}\left(\sum_{d=1}^D \A^d\x^*\right) \notag \\
= &\frac{\partial F}{\partial \x^d}(\x^*,\ldots,\x^*) + \A^{d\top}\y^*.
\end{align}
Note that the last equality follows from~\eqref{eq:equality-constraint}. Plugging the above into the last inequality we obtain
\begin{multline}
\left(\frac{\partial F}{\partial \x^d}(\x^*,\ldots,\x^*) + \A^{d\top}\y^*\right)^\top(\x^d - \x^*) \ge 0\\ \forall \x^d\in\cX^d,\forall d,
\end{multline}
which is exactly~\eqref{eq:kkt-appendix-inequality}, and this completes the proof.

\subsection{Proof of Proposition~\ref{propos:KKT-stationary}}
Let $(\x^*,\ldots,\x^*,\y^*)$ be a KKT point of~\eqref{prob:mrf-decomposed}. We have seen in the previous proof that
\begin{multline}\label{eq:kkt-appendix-inequality-2}
\left(\frac{\partial F}{\partial \x^d}(\x^*,\ldots,\x^*) + \A^{d\top}\y^*\right)^\top(\x^d-\x^*) \ge 0\\ \forall\x^d\in\cX^d,\forall d.
\end{multline}

According to~\eqref{eq:energy-factor-linear}:
\begin{equation}
\frac{\partial F}{\partial \x^d}(\x^1,\ldots,\x^D) = \p^d,
\end{equation}
where $\p^d$ is defined by~\eqref{eq:p}. Now let $\p^{*d}$ be the value of $\p^d$ where $(\x^1,\ldots,\x^D)$ is replaced by $(\x^*,\ldots,\x^*)$, \ie $\p^{*d} = (\p_1^{*d},\ldots,\p_n^{*d})$ where
\begin{multline}\label{eq:p-star}
\p_i^{*d} = \sum_{\alpha=d}^D \left( \sum_{i_1\ldots i_{d-1}ii_{d+1}\ldots i_\alpha \in\cC} \F_{i_1i_2\ldots i_\alpha} \bigotimes \right. \\ \left.\set{\x^*_{i_1},\ldots,\x^*_{i_{d-1}}, \x^*_{i_{d+1}},\ldots,\x^*_{i_\alpha}} \vphantom{\sum_{C=i_1\ldots i_{d-1}ii_{d+1}\ldots i_\alpha \in\cC}}\right) \forall i\in\cV.
\end{multline}
Notice that $\frac{\partial F}{\partial \x^d}(\x^*,\ldots,\x^*) = \p^{*d}$,~\eqref{eq:kkt-appendix-inequality-2} becomes
\begin{equation}
\left(\p^{*d} + \A^{d\top}\y^*\right)^\top(\x^d-\x^*) \ge 0\quad \forall\x^d\in\cX^d,\forall d.
\end{equation}
According to~\eqref{eq:set-intersection} we have $\cX\subseteq\cX^d$ and therefore the above inequality implies
\begin{multline}
\left(\p^{*d} + \A^{d\top}\y^*\right)^\top(\x-\x^*) \ge 0\quad \forall\x\in\cX,\forall d.
\end{multline}
Summing this inequality for all $d$ we get
\begin{multline}
\left(\sum_{d=1}^D \p^{*d}\right)^\top(\x-\x^*) \\ + \y^{*\top}\left(\sum_{d=1}^D\A^d\right)(\x-\x^*) \ge 0 \quad \forall\x\in\cX.
\end{multline}
Yet, according to~\eqref{eq:equality-constraint} we have $\sum_{d=1}^D\A^d\x = \sum_{d=1}^D\A^d\x^* = \mathbf{0}$. Therefore, the second term in the above inequality is $0$, yielding
\begin{equation}
\left(\sum_{d=1}^D \p^{*d}\right)^\top(\x-\x^*) \ge 0 \quad \forall\x\in\cX.
\end{equation}
Now if we can prove that
\begin{equation}\label{eq:p-star-sum}
\sum_{d=1}^D \p^{*d} = \nabla E(\x^*),
\end{equation}
then we have $\nabla E(\x^*)^\top(\x-\x^*) \ge 0 \ \forall\x\in\cX$ and thus according to Definition~\ref{def:stationary}, $\x^*$ is a stationary point of~\eqref{prob:relax}. 

Let us now prove~\eqref{eq:p-star-sum}. Indeed, we can rewrite~\eqref{eq:p-star} as
\begin{multline}\label{eq:p-star-2}
\p_i^{*d} = \sum_{\alpha=d}^D \left( \sum_{\substack{C\in\cC \\ C = (i_1\ldots i_{d-1}ii_{d+1}\ldots i_\alpha)}} \F_{C}\bigotimes \set{\x^*_{j}}_{j\in C\setminus i}\right) \\ \forall i\in\cV.
\end{multline}
Therefore,
\begin{multline}\label{eq:p-star-sum-2}
\sum_{d=1}^D \p_i^{*d} = \sum_{d=1}^D  \sum_{\alpha=d}^D \sum_{\substack{C\in\cC \\ C = (i_1\ldots i_{d-1}ii_{d+1}\ldots i_\alpha)}} \F_{C}\bigotimes \set{\x^*_{j}}_{j\in C\setminus i} \\ \forall i\in\cV.
\end{multline}
Let's take a closer look at this triple sum. The double sum $$\sum_{\alpha=d}^D \sum_{\substack{C\in\cC \\ C = (i_1\ldots i_{d-1}ii_{d+1}\ldots i_\alpha)}}$$ basically means \emph{iterating through all cliques whose sizes are $\ge d$ and whose $d$\textsuperscript{th} node is $i$}. Obviously the condition ``sizes $\ge d$'' is redundant here, thus the above means \emph{iterating through all cliques whose $d$\textsuperscript{th} node is $i$}. Combined with $\sum_{d=1}^D$, the above triple sum means \emph{for each size $d$, iterating through all cliques whose $d$\textsuperscript{th} node is $i$}, which is clearly equivalent to \emph{iterating through all cliques that contain $i$}. Therefore,~\eqref{eq:p-star-sum-2} can be rewritten more compactly as
\begin{equation}
\sum_{d=1}^D \p_i^{*d} = \sum_{C\in\cC(i)} \F_{C}\bigotimes \set{\x^*_{j}}_{j\in C\setminus i} \quad \forall i\in\cV,
\end{equation}
where $\cC(i)$ is the set of cliques that contain the node $i$. Recall from~\eqref{eq:mrf-energy-partial-derivative} that the last expression is actually $\frac{\partial E(\x^*)}{\partial \x^i}$, \ie
\begin{equation}
\sum_{d=1}^D \p_i^{*d} = \frac{\partial E(\x^*)}{\partial \x^i} \quad \forall i\in\cV,
\end{equation}
or equivalently
\begin{equation}
\sum_{d=1}^D \p^{*d} = \nabla E(\x^*),
\end{equation}
which is~\eqref{eq:p-star-sum}, and this completes the proof.

\section{More details on the implemented methods}\label{append:methods}
We present additional details on PGD, FW, ADMM as well as CQP (we omit BCD since it was presented with sufficient details in the paper).

Recall that our nonconvex relaxation is to minimize
\begin{equation}\label{eq:mrf-Energy-appendix}
E(\x) = \sum_{C\in\cC} \F_{C} \bigotimes\set{\x_i}_{i\in C}\tag{10}
\end{equation}
subject to $\x\in\cX := \set{\x \ \Big| \ \mathbf{1}^\top\x_i = 1, \x_i\ge\mathbf{0}	\ \forall i\in\cV}.$

\subsection{PGD and FW}
Recall from Section~\ref{sec:gradient-methods} that the main update steps in PGD and FW are respectively
\begin{equation}\label{eq:update-pgd}
\s^{(k)}= \argmin_{\s\in\cX} \norm{\x^{(k)} - \beta^{(k)}\nabla E(\x^{(k)}) - \s}^2,
\end{equation}
and
\begin{equation}\label{eq:update-fw}
\s^{(k)}= \argmin_{\s\in\cX} \s^\top\nabla E(\x^{(k)}).
\end{equation}

Clearly, in the PGD update step~\eqref{eq:update-pgd} the vector $\s^{(k)}$ is the projection of $\x^{(k)} - \beta^{(k)}\nabla E(\x^{(k)})$ onto $\cX$. As we have discussed at the end of Section~\eqref{sec:admm}, this projection is reduced to independent projections onto the simplex $\set{\x_i \ | \ \mathbf{1}^\top\x_i = 1, \x_i\ge\mathbf{0}}$ for each node $i$. In our implementation we used the method introduced in~\cite{condat2016fast} for this simplex projection task.

The FW update step~\eqref{eq:update-fw} can be solved independently for each node as well:
\begin{equation}
\s_i^{(k)}= \argmin_{\mathbf{1}^\top\s_i = 1, \s_i\ge\mathbf{0}} \s_i^\top\frac{\partial E(\x^{(k)})}{\partial \x_i}\quad\forall i\in\cV,
\end{equation}
which is similar to the BCD update step~\eqref{eq:BCD-update} and thus can be solved using Lemma~\ref{lem:bcd}.

Next, we describe the line-search procedure~\eqref{eq:line-search} for these methods. Before going into details, we should note that in addition to line-search, we also implemented other step-size update rules such as \emph{diminishing} or \emph{Armijo} ones. However, we found that these rules do not work as well as line-search (the diminishing rule converges slowly while the search in the Armijo rule is expensive). We refer to~\cite{bertsekas1999nonlinear} (Chapter 2) for further details on these rules.

\paragraph{Line search}
The line-search step consists of finding
\begin{equation}\label{eq:line-search-appendix}
\alpha^{(k)} = \argmin_{0\le \alpha\le 1}E\left(\x^{(k)} + \alpha\r^{(k)}\right),\end{equation}
where $\r^{(k)} = \s^{(k)} - \x^{(k)}$. The term $E\left(\x^{(k)} + \alpha\r^{(k)}\right)$ is clearly a $D$\textsuperscript{th}-degree polynomial of $\alpha$ (recall that $D$ is the degree of the MRF), which we denote $p(\alpha)$. If we can determine the coefficients of $p(\alpha)$, then~\eqref{eq:line-search-appendix} can be solved efficiently. In particular, if $D \le 3$ then~\eqref{eq:line-search-appendix} has simple closed-form solutions (since the derivative of a $3$\textsuperscript{rd}-order polynomial is a $2$\textsuperscript{nd}-order one, which has simple closed-form solutions). For $D > 3$ we find that it is efficient enough to perform an exhaustive search over the interval $[0,1]$ (with some increment value $\delta$) for the best value of $\alpha$. In the implementation we used $\delta = 0.0001$.

Now let us describe how to find the coefficients of $p(\alpha)$.

For pairwise MRFs (\ie $D=2$), the energy is
\begin{equation}\label{eq:energy-pairwise}
	E_\pw(\x) = \sum_{i\in \cV} \F_i^\top \x_i + \sum_{ij\in\cE} \x_i^\top\F_{ij}\x_j,
\end{equation}
where $\cE$ is the set of edges, and thus
\begin{align}
	&p(\alpha) = E_\pw(\x + \alpha\r) \\
	= &\sum_{i\in \cV} \F_i^\top (\x_i + \alpha\r_i) + \sum_{ij\in\cE} (\x_i + \alpha\r_i)^\top\F_{ij}(\x_j + \alpha\r_j)\\
	= &A\alpha^2 + B\alpha + C,
\end{align}
where
\begin{align}
A &= \sum_{ij\in\cE} \r_i^\top\F_{ij}\r_j \\
B &= \sum_{i\in \cV} \F_i^\top \r_i + \sum_{ij\in\cE}\left( \x_i^\top\F_{ij}\r_j + \r_i^\top\F_{ij}\x_j\right)\\
C &= E_\pw(\x).
\end{align}

For higher-order MRFs, the analytical expressions of the polynomial coefficients are very complicated. Instead, we can find them numerically as follows. Since $p(\alpha)$ is a $D$\textsuperscript{th}-degree polynomial, it has $D+1$ coefficients, where the constant coefficient is already known:
\begin{equation}
p(0) = E(\x^{(k)}).
\end{equation}
It remains $D$ unknown coefficients, which can be computed if we have $D$ equations. Indeed, if we evaluate $p(\alpha)$ at $D$ different random values of $\alpha$ (which must be different than $0$), then we obtain $D$ linear equations whose variables are the coefficients of $p(\alpha)$. Solving this system of linear equations we get the values of these coefficients. This procedure requires $D$ evaluations of the energy $E\left(\x^{(k)} + \alpha\r^{(k)}\right)$, but we find that it is efficient enough in practice.

\subsection{Convex QP relaxation}
This relaxation was presented in~\cite{ravikumar2006quadratic} for pairwise MRFs~\eqref{eq:energy-pairwise}. Define:
\begin{equation}
d_i(s) = \sum_{j\in\cN(i)}\sum_{t\in\cS_j} \frac{1}{2}\abs{f_{ij}(s,t)}.
\end{equation}
Denote $\d_i = (d_i(s))_{s\in\cS_i}$ and $\D_i = \diag(\d_i)$, the diagonal matrix composed by $\d_i$. The convex QP relaxation energy is given by
\begin{align}
E_\text{cqp}(\x) &= E_\pw(\x) -\sum_{i\in \cV} \d_i^\top \x_i + \sum_{i\in \cV} \x_i^\top \D_i \x_i.
\end{align}
This convex energy can be minimized using different methods. Here we propose to solve it using Frank-Wolfe algorithm, which has the guarantee to reach the global optimum.

Similarly to the previous nonconvex Frank-Wolfe algorithm, the update step~\eqref{eq:update-fw} can be solved using Lemma~\ref{lem:bcd}, and the line-search has closed-form solutions:
\begin{align}
E_\text{cqp}(\x + \alpha\r) = &E_\pw(\x + \alpha\r) - \sum_{i\in \cV} \d_i^\top (\x_i + \alpha\r_i)\notag \\
&+ \sum_{i\in \cV} (\x_i + \alpha\r_i)^\top \D_i (\x_i + \alpha\r_i) \\
= &A'\alpha^2 + B'\alpha + C',
\end{align}
where
\begin{align}
A' &= A + \sum_{i\in\cV} \r_i^\top\D_i\r_i \\
B' &= B + \sum_{i\in\cV}\left(-\d_i^\top\r_i + \r_i^\top\D_i\x_i + \x_i^\top\D_i\r_i\right)\\
C' &= C + \sum_{i\in\cV}\left(-\d_i^\top\x_i + \x_i^\top\D_i\x_i\right).
\end{align}

\subsection{ADMM}\label{append:admm}
In this section, we give more details on the instantiation of ADMM into different decompositions. As we have seen in Section~\ref{sec:admm}, there is an infinite number of such decompositions. Some examples include:
\begin{alignat}{3}
&\text{(cyclic)}&&\x^{d-1}	&&=\x^d, \quad d = 2,\ldots, D,\label{eq:cyclic-append}\\
&\text{(star)}&&\x^1 		&&=\x^d,\quad d = 2,\ldots, D,\label{eq:star-append}\\
&\text{(symmetric)}\quad&&\x^d 		&&=(\x^1+\cdots+\x^D)/D\quad \forall d.\label{eq:symmetric-append}
\end{alignat}
Let us consider for example the \emph{cyclic} decomposition. We obtain the following problem, equivalent to~\eqref{prob:relax}:
\begin{equation}\label{prob:mrf-decomposed-cyclic}
\begin{aligned} 
\mbox{min}\quad & F(\x^1,\x^2,\ldots,\x^D)\\ 
\mbox{s.t.}\quad & \x^{d-1} = \x^d,\quad d=2,\ldots,D,\\
&\x^d\in\cX^d, \qquad d=1,\ldots,D,
\end{aligned}
\end{equation}
where $\cX^1,\ldots,\cX^D$ are closed convex sets satisfying $\cX^1\cap\cX^2\cap\cdots\cap\cX^D = \cX$, and $F$ is defined by~\eqref{eq:energy-decomposed}.

The augmented Lagrangian of this problem is:
\begin{multline}\label{eq:lagrangian-cyclic}
	L_\rho(\x^1,\ldots,\x^D,\y) = F(\x^1,\ldots,\x^D) \\ + \sum_{d=2}^D \inner{\y^d,\x^{d-1} - \x^d} + \frac{\rho}{2}\sum_{d=2}^D\norm{\x^{d-1} - \x^d}^2,
\end{multline}
where $\y=(\y^2,\ldots,\y^D)$. The $\y$ update~\eqref{eq:update-y} becomes
\begin{equation}\label{eq:update-y-cyclic}
		\y^{d^{(k+1)}} = \y^{d^{(k)}} + \rho\left(\x^{d-1^{(k+1)}} - \x^{d^{(k+1)}}\right).
	\end{equation}
Consider the $\x$ update~\eqref{eq:update-x}. Plugging~\eqref{eq:energy-factor-linear} into~\eqref{eq:lagrangian-cyclic}, expanding and regrouping, we obtain that $L_\rho(\x^1,\ldots,\x^D,\y)$ is equal to each of the following expressions:
\begin{align}
&\frac{\rho}{2}\norm{\x^1}^2 - \inner{\x^1, \rho\x^2-\y^2-\p^1} + \mathrm{cst}(\x^1),\\
&\rho\norm{\x^d}^2 - \inner{\x^d,\rho\x^{d-1} + \rho\x^{d+1} + \y^d - \y^{d+1} -\p^d } \notag\\ 
&\qquad\qquad\qquad + \mathrm{cst}(\x^d) \qquad(2\le d\le D-1),\\
&\frac{\rho}{2}\norm{\x^D}^2 - \inner{\x^D, \rho\x^{D-1}+\y^D-\p^D} + \mathrm{cst}(\x^D).
\end{align}
From this, it is straightforward to see that the $\x$ update~\eqref{eq:update-x} is reduced to~\eqref{eq:projection} where $(\c_d)_{1\le d \le D}$ are defined by~\eqref{eq:cyclic-c1},~\eqref{eq:cyclic-cd} and~\eqref{eq:cyclic-cD}.

It is straightforward to obtain similar results for the other decompositions.

\begin{figure*}[!htb]
 \centering
     \begin{subfigure}[t]{0.18\linewidth}
     \centering
         \includegraphics[width=\linewidth]{cones_disp2}\vspace{-5pt}
         \caption{\scriptsize Ground-truth}
     \end{subfigure}%
     ~
     \begin{subfigure}[t]{0.18\linewidth}
     \centering
             \includegraphics[width=\linewidth]{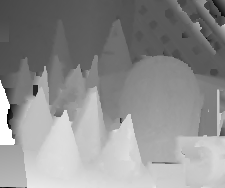}\vspace{-5pt}
             \caption{\scriptsize $\alpha$-Fusion (18582.85)}
         \end{subfigure}%
     ~
     \begin{subfigure}[t]{0.18\linewidth}
     \centering
             \includegraphics[width=\linewidth]{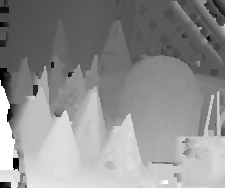}\vspace{-5pt}
             \caption{\scriptsize TRBP (18640.25)}
         \end{subfigure}%
     ~
     \begin{subfigure}[t]{0.18\linewidth}
    \centering
            \includegraphics[width=\linewidth]{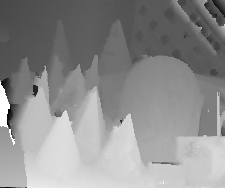}\vspace{-5pt}
            \caption{\scriptsize AD3 (18763.13)}
        \end{subfigure}%
     ~
     \begin{subfigure}[t]{0.18\linewidth}
     \centering
             \includegraphics[width=\linewidth]{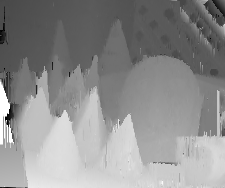}\vspace{-5pt}
             \caption{\scriptsize BUNDLE (20055.65)}
         \end{subfigure}\\%
     \begin{subfigure}[t]{0.18\linewidth}
     \centering
             \includegraphics[width=\linewidth]{cones_small_srmp}\vspace{-5pt}
             \caption{\scriptsize SRMP (\textbf{18433.01})}
         \end{subfigure}%
     ~
     \begin{subfigure}[t]{0.18\linewidth}
     \centering
             \includegraphics[width=\linewidth]{cones_small_bcd0}\vspace{-5pt}
             \caption{\scriptsize BCD (18926.70)}
         \end{subfigure}%
     ~
     \begin{subfigure}[t]{0.18\linewidth}
     \centering
             \includegraphics[width=\linewidth]{cones_small_fw0}\vspace{-5pt}
             \caption{\scriptsize FW (18776.26)}
         \end{subfigure}%
     ~
     \begin{subfigure}[t]{0.18\linewidth}
    \centering
            \includegraphics[width=\linewidth]{cones_small_pgd0}\vspace{-5pt}
            \caption{\scriptsize PGD (19060.17)}
        \end{subfigure}%
     ~
     \begin{subfigure}[t]{0.18\linewidth}
     \centering
             \includegraphics[width=\linewidth]{cones_small_admm0}\vspace{-5pt}
             \caption{\scriptsize ADMM (18590.87)}
         \end{subfigure}
     \caption{\label{fig:stereo-full}Resulted disparity maps and energy values using second-order MRFs for the \emph{cones} scene of the Middlebury stereo dataset~\cite{scharstein2003high}.}
 \end{figure*}

\section{Details on the experiments}\label{append:experiments}

We replicated the model presented in~\cite{woodford2009global} for the second-order stereo experiment, with some simplifications: we only used segmentation proposals (denoted by \emph{SegPln} in~\cite{woodford2009global}) and omitted the binary visibility variables and edges, so that all the nodes have the same number of labels. We ran the code provided by~\cite{woodford2009global} to get the unary potentials as well as the 14 proposals, and then built the MRF model using OpenGM~\cite{opengm-library}. An example of resulted disparity maps for the \emph{cones} scene of the Middlebury stereo dataset~\cite{scharstein2003high} is given in Figure~\ref{fig:stereo-full}.

For further details on the other modes, we refer to~\cite{kappes2015ijcv}.

The detailed results of the experiments are provided at the end of this document.

\onecolumn
{\scriptsize\begin{longtable}{|l|l|r|r|r|r|r|r|r|r|} 
\caption[]{inpainting-n4}
\\\hline
\textbf{inpainting-n4}& & FastPD& $\alpha$-Exp& TRBP& ADDD& MPLP& MPLP-C& TRWS& BUNDLE\\\hline
\endfirsthead
\hline
\textbf{inpainting-n4}& & FastPD& $\alpha$-Exp& TRBP& ADDD& MPLP& MPLP-C& TRWS& BUNDLE\\\hline
\endhead
triplepoint4-plain-ring-inverse& value& 424.90& \bf 424.12& 475.95& 482.23& 508.94& 453.17& 496.37& 425.90\\\nopagebreak
& bound& 205.21& -Inf& -Inf& 402.83& 339.29& 411.48& 411.59& \bf 411.87\\\nopagebreak
& runtime& 0.03& 0.02& 33.04& 28.59& 1.75& 3615.97& 2.15& 44.41\\\hline
triplepoint4-plain-ring& value& \bf 484.59& \bf 484.59& \bf 484.59& \bf 484.59& 485.38& \bf 484.59& \bf 484.59& \bf 484.59\\\nopagebreak
& bound& 384.57& -Inf& -Inf& \bf 484.59& 484.58& 484.59& \bf 484.59& 484.59\\\nopagebreak
& runtime& 0.03& 0.02& 13.85& 3.15& 108.89& 118.43& 0.59& 27.96\\\hline
\hline
\textbf{mean energy} & & 454.75& 454.35& 480.27& 483.41& 497.16& 468.88& 467.70& 455.25\\\hline
\textbf{mean bound} & & 294.89& -Inf& -Inf& 443.71& 411.94& 448.03& 448.09& 448.23\\\hline
\textbf{mean runtime} & & 0.03& 0.02& 23.45& 15.87& 55.32& 1867.20& 1.37& 36.18\\\hline
\textbf{best value} & & 50.00& 100.00& 50.00& 50.00& 0.00& 50.00& 50.00& 50.00\\\hline
\textbf{best bound} & & 0.00& 0.00& 0.00& 50.00& 0.00& 0.00& 50.00& 50.00\\\hline
\textbf{verified opt} & & 0.00& 0.00& 0.00& 50.00& 0.00& 0.00& 50.00& 0.00\\\hline
\end{longtable}}

{\scriptsize\begin{longtable}{|l|l|r|r|r|r|r|} 
\caption[]{inpainting-n4}
\\\hline
\textbf{inpainting-n4}& & CQP& ADMM& BCD& FW& PGD\\\hline
\endfirsthead
\hline
\textbf{inpainting-n4}& & CQP& ADMM& BCD& FW& PGD\\\hline
\endhead
triplepoint4-plain-ring-inverse& value& 2256.45& \bf 424.12& 443.18& 443.18& 444.75\\\nopagebreak
& bound& -Inf& -Inf& -Inf& -Inf& -Inf\\\nopagebreak
& runtime& 2.60& 7.69& 0.11& 1.05& 0.77\\\hline
triplepoint4-plain-ring& value& 542.57& \bf 484.59& 528.57& 533.29& 534.86\\\nopagebreak
& bound& -Inf& -Inf& -Inf& -Inf& -Inf\\\nopagebreak
& runtime& 1.24& 12.00& 0.11& 1.15& 0.85\\\hline
\hline
\textbf{mean energy} & & 490.09& 454.35& 485.88& 488.23& 489.80\\\hline
\textbf{mean bound} & & -Inf& -Inf& -Inf& -Inf& -Inf\\\hline
\textbf{mean runtime} & & 1.92& 9.84& 0.11& 1.10& 0.81\\\hline
\textbf{best value} & & 0.00& 100.00& 0.00& 0.00& 0.00\\\hline
\textbf{best bound} & & 0.00& 0.00& 0.00& 0.00& 0.00\\\hline
\textbf{verified opt} & & 0.00& 0.00& 0.00& 0.00& 0.00\\\hline
\end{longtable}}

{\scriptsize\begin{longtable}{|l|l|r|r|r|r|r|r|r|r|} 
\caption[]{inpainting-n8}
\\\hline
\textbf{inpainting-n8}& & $\alpha$-Exp& FastPD& TRBP& ADDD& MPLP& MPLP-C& BUNDLE& TRWS\\\hline
\endfirsthead
\hline
\textbf{inpainting-n8}& & $\alpha$-Exp& FastPD& TRBP& ADDD& MPLP& MPLP-C& BUNDLE& TRWS\\\hline
\endhead
triplepoint4-plain-ring-inverse& value& 434.84& 434.84& 496.40& 714.42& 442.42& 463.88& 435.32& 504.97\\\nopagebreak
& bound& -Inf& 0.00& -Inf& 406.71& 412.37& 413.49& \bf 415.83& 413.20\\\nopagebreak
& runtime& 0.90& 0.19& 97.95& 57.01& 1107.98& 3660.44& 112.91& 16.09\\\hline
triplepoint4-plain-ring& value& \bf 495.20& \bf 495.20& \bf 495.20& 495.85& 495.52& \bf 495.20& \bf 495.20& \bf 495.20\\\nopagebreak
& bound& -Inf& 272.56& -Inf& 495.18& 494.72& \bf 495.20& 495.04& 494.71\\\nopagebreak
& runtime& 0.67& 0.11& 30.04& 14.56& 581.96& 884.34& 110.56& 16.37\\\hline
\hline
\textbf{mean energy} & & 465.02& 465.02& 494.02& 605.14& 468.83& 469.78& 465.26& 466.80\\\hline
\textbf{mean bound} & & -Inf& 136.28& -Inf& 450.95& 453.55& 454.35& 455.43& 453.96\\\hline
\textbf{mean runtime} & & 0.78& 0.15& 64.00& 35.78& 844.97& 2272.39& 111.74& 16.23\\\hline
\textbf{best value} & & 50.00& 50.00& 50.00& 0.00& 0.00& 50.00& 50.00& 50.00\\\hline
\textbf{best bound} & & 0.00& 0.00& 0.00& 0.00& 0.00& 50.00& 50.00& 0.00\\\hline
\textbf{verified opt} & & 0.00& 0.00& 0.00& 0.00& 0.00& 0.00& 0.00& 0.00\\\hline
\end{longtable}}

{\scriptsize\begin{longtable}{|l|l|r|r|r|r|r|} 
\caption[]{inpainting-n8}
\\\hline
\textbf{inpainting-n8}& & CQP& ADMM& BCD& FW& PGD\\\hline
\endfirsthead
\hline
\textbf{inpainting-n8}& & CQP& ADMM& BCD& FW& PGD\\\hline
\endhead
triplepoint4-plain-ring-inverse& value& 1819.57& \bf 434.32& 438.95& 446.19& 446.19\\\nopagebreak
& bound& -Inf& -Inf& -Inf& -Inf& -Inf\\\nopagebreak
& runtime& 20.78& 38.71& 0.31& 6.33& 6.55\\\hline
triplepoint4-plain-ring& value& 538.25& \bf 495.20& 524.94& 533.45& 533.45\\\nopagebreak
& bound& -Inf& -Inf& -Inf& -Inf& -Inf\\\nopagebreak
& runtime& 2.46& 42.57& 0.28& 5.55& 3.83\\\hline
\hline
\textbf{mean energy} & & 489.82& 464.76& 481.95& 489.82& 489.82\\\hline
\textbf{mean bound} & & -Inf& -Inf& -Inf& -Inf& -Inf\\\hline
\textbf{mean runtime} & & 11.62& 40.64& 0.29& 5.94& 5.19\\\hline
\textbf{best value} & & 0.00& 100.00& 0.00& 0.00& 0.00\\\hline
\textbf{best bound} & & 0.00& 0.00& 0.00& 0.00& 0.00\\\hline
\textbf{verified opt} & & 0.00& 0.00& 0.00& 0.00& 0.00\\\hline
\end{longtable}}

{\scriptsize\begin{longtable}{|l|l|r|r|r|r|r|r|r|r|} 
\caption[]{matching}
\\\hline
\textbf{matching}& & TRBP& ADDD& MPLP& MPLP-C& BUNDLE& TRWS& CQP& ADMM\\\hline
\endfirsthead
\hline
\textbf{matching}& & TRBP& ADDD& MPLP& MPLP-C& BUNDLE& TRWS& CQP& ADMM\\\hline
\endhead
matching0& value& 60000000075.71& 200000000047.27& 90000000059.69& \bf 19.36& 58.64& 61.05& 118.90& 42.09\\\nopagebreak
& bound& -Inf& 11.56& 10.96& \bf 19.36& 11.27& 11.02& -Inf& -Inf\\\nopagebreak
& runtime& 0.00& 2.45& 0.22& 8.02& 1.09& 0.04& 0.06& 0.02\\\hline
matching1& value& 170000000090.50& 70000000031.36& 50000000030.34& \bf 23.58& 10000000021.89& 102.20& 138.99& 107.31\\\nopagebreak
& bound& -Inf& 20.13& 18.47& \bf 23.58& 17.48& 18.52& -Inf& -Inf\\\nopagebreak
& runtime& 0.00& 3.82& 0.52& 4.52& 2.70& 0.04& 0.10& 0.94\\\hline
matching2& value& 110000000096.00& 20000000026.59& 30000000025.18& \bf 26.08& 20000000043.93& 51.59& 156.46& 107.41\\\nopagebreak
& bound& -Inf& 22.97& 21.07& \bf 26.08& 19.87& 21.18& -Inf& -Inf\\\nopagebreak
& runtime& 0.00& 4.12& 0.94& 8.25& 3.56& 0.12& 0.08& 0.26\\\hline
matching3& value& 80000000066.03& 130000000051.70& 90000000051.81& \bf 15.86& 10000000042.82& 41.92& 93.67& 43.69\\\nopagebreak
& bound& -Inf& 10.72& 10.15& \bf 15.86& 9.25& 10.14& -Inf& -Inf\\\nopagebreak
& runtime& 0.00& 2.25& 0.21& 3.36& 1.96& 0.01& 0.07& 0.02\\\hline
\hline
\textbf{mean energy} & & 97500000064.52& 105000000039.23& 65000000041.76& 21.22& 10000000041.82& 63.52& 127.01& 75.12\\\hline
\textbf{mean bound} & & -Inf& 16.35& 15.16& 21.22& 14.47& 15.22& -Inf& -Inf\\\hline
\textbf{mean runtime} & & 0.00& 3.16& 0.47& 6.04& 2.33& 0.05& 0.08& 0.31\\\hline
\textbf{best value} & & 0.00& 0.00& 0.00& 100.00& 0.00& 0.00& 0.00& 0.00\\\hline
\textbf{best bound} & & 0.00& 0.00& 0.00& 100.00& 0.00& 0.00& 0.00& 0.00\\\hline
\textbf{verified opt} & & 0.00& 0.00& 0.00& 100.00& 0.00& 0.00& 0.00& 0.00\\\hline
\end{longtable}}

{\scriptsize\begin{longtable}{|l|l|r|r|r|} 
\caption[]{matching}
\\\hline
\textbf{matching}& & BCD& FW& PGD\\\hline
\endfirsthead
\hline
\textbf{matching}& & BCD& FW& PGD\\\hline
\endhead
matching0& value& 43.61& 56.10& 49.45\\\nopagebreak
& bound& -Inf& -Inf& -Inf\\\nopagebreak
& runtime& 0.00& 0.19& 8.08\\\hline
matching1& value& 118.00& 77.31& 79.01\\\nopagebreak
& bound& -Inf& -Inf& -Inf\\\nopagebreak
& runtime& 0.00& 23.66& 21.36\\\hline
matching2& value& 139.74& 89.46& 62.40\\\nopagebreak
& bound& -Inf& -Inf& -Inf\\\nopagebreak
& runtime& 0.00& 55.74& 19.28\\\hline
matching3& value& 38.09& 43.98& 43.21\\\nopagebreak
& bound& -Inf& -Inf& -Inf\\\nopagebreak
& runtime& 0.00& 0.81& 4.11\\\hline
\hline
\textbf{mean energy} & & 84.86& 66.71& 58.52\\\hline
\textbf{mean bound} & & -Inf& -Inf& -Inf\\\hline
\textbf{mean runtime} & & 0.00& 20.10& 13.21\\\hline
\textbf{best value} & & 0.00& 0.00& 0.00\\\hline
\textbf{best bound} & & 0.00& 0.00& 0.00\\\hline
\textbf{verified opt} & & 0.00& 0.00& 0.00\\\hline
\end{longtable}}

{\scriptsize\begin{longtable}{|l|l|r|r|r|r|r|r|r|r|} 
\caption[]{mrf-stereo}
\\\hline
\textbf{mrf-stereo}& & FastPD& $\alpha$-Exp& TRBP& ADDD& MPLP& MPLP-C& TRWS& BUNDLE\\\hline
\endfirsthead
\hline
\textbf{mrf-stereo}& & FastPD& $\alpha$-Exp& TRBP& ADDD& MPLP& MPLP-C& TRWS& BUNDLE\\\hline
\endhead
ted-gm& value& 1344017.00& \bf 1343176.00& 1460166.00& NaN& NaN& NaN& 1346202.00& 1563172.00\\\nopagebreak
& bound& 395613.00& -Inf& -Inf& NaN& NaN& NaN& \bf 1337092.22& 1334223.01\\\nopagebreak
& runtime& 14.94& 29.75& 3616.74& NaN& NaN& NaN& 391.34& 3530.00\\\hline
tsu-gm& value& 370825.00& 370255.00& 411157.00& 455874.00& 369304.00& 369865.00& 369279.00& \bf 369218.00\\\nopagebreak
& bound& 31900.00& -Inf& -Inf& 299780.16& 367001.47& 366988.29& 369217.58& \bf 369218.00\\\nopagebreak
& runtime& 1.72& 3.64& 1985.50& 1066.79& 4781.02& 4212.26& 393.76& 670.81\\\hline
ven-gm& value& 3127923.00& 3138157.00& 3122190.00& NaN& NaN& NaN& \bf 3048404.00& 3061733.00\\\nopagebreak
& bound& 475665.00& -Inf& -Inf& NaN& NaN& NaN& \bf 3047929.95& 3047785.37\\\nopagebreak
& runtime& 4.76& 10.87& 2030.13& NaN& NaN& NaN& 478.49& 1917.58\\\hline
\hline
\textbf{mean energy} & & 1614255.00& 1617196.00& 1664504.33& NaN& NaN& NaN& 1587596.67& 1664707.67\\\hline
\textbf{mean bound} & & 301059.33& -Inf& -Inf& NaN& NaN& NaN& 1584746.58& 1583742.13\\\hline
\textbf{mean runtime} & & 7.14& 14.75& 2544.12& NaN& NaN& NaN& 421.20& 2039.47\\\hline
\textbf{best value} & & 0.00& 33.33& 0.00& 0.00& 0.00& 0.00& 33.33& 33.33\\\hline
\textbf{best bound} & & 0.00& 0.00& 0.00& 0.00& 0.00& 0.00& 66.67& 33.33\\\hline
\textbf{verified opt} & & 0.00& 0.00& 0.00& 0.00& 0.00& 0.00& 0.00& 0.00\\\hline
\end{longtable}}

{\scriptsize\begin{longtable}{|l|l|r|r|r|r|r|} 
\caption[]{mrf-stereo}
\\\hline
\textbf{mrf-stereo}& & CQP& ADMM& BCD& FW& PGD\\\hline
\endfirsthead
\hline
\textbf{mrf-stereo}& & CQP& ADMM& BCD& FW& PGD\\\hline
\endhead
ted-gm& value& 4195611.00& 1373030.00& 3436281.00& 3020579.00& 2694493.00\\\nopagebreak
& bound& -Inf& -Inf& -Inf& -Inf& -Inf\\\nopagebreak
& runtime& 3602.97& 3628.80& 15.64& 1740.10& 2109.65\\\hline
tsu-gm& value& 3621062.00& 375954.00& 2722934.00& 2352499.00& 2114223.00\\\nopagebreak
& bound& -Inf& -Inf& -Inf& -Inf& -Inf\\\nopagebreak
& runtime& 3600.79& 807.70& 5.33& 622.64& 120.38\\\hline
ven-gm& value& 26408665.00& 3123334.00& 14907352.00& 13114176.00& 10818561.00\\\nopagebreak
& bound& -Inf& -Inf& -Inf& -Inf& -Inf\\\nopagebreak
& runtime& 3602.28& 2696.49& 11.48& 3604.63& 2298.42\\\hline
\hline
\textbf{mean energy} & & 11408446.00& 1624106.00& 7022189.00& 6162418.00& 5209092.33\\\hline
\textbf{mean bound} & & -Inf& -Inf& -Inf& -Inf& -Inf\\\hline
\textbf{mean runtime} & & 3602.01& 2377.66& 10.82& 1989.12& 1509.49\\\hline
\textbf{best value} & & 0.00& 0.00& 0.00& 0.00& 0.00\\\hline
\textbf{best bound} & & 0.00& 0.00& 0.00& 0.00& 0.00\\\hline
\textbf{verified opt} & & 0.00& 0.00& 0.00& 0.00& 0.00\\\hline
\end{longtable}}

{\scriptsize\begin{longtable}{|l|l|r|r|r|r|r|r|r|r|} 
\caption[]{inclusion}
\\\hline
\textbf{inclusion}& & $\alpha$-Fusion& TRBP& ADDD& MPLP& MPLP-C& BUNDLE& SRMP& ADMM\\\hline
\endfirsthead
\hline
\textbf{inclusion}& & $\alpha$-Fusion& TRBP& ADDD& MPLP& MPLP-C& BUNDLE& SRMP& ADMM\\\hline
\endhead
modelH-1-0.8-0.2& value& 1595.06& 1416.07& 2416.58& 3416.08& 5415.89& 5427.91& \bf 1415.94& \bf 1415.94\\\nopagebreak
& bound& -Inf& -Inf& 1415.71& 1415.70& 1415.71& 1406.09& \bf 1415.94& -Inf\\\nopagebreak
& runtime& 0.06& 21.93& 10.52& 12.17& 3843.79& 99.77& 0.11& 106.70\\\hline
modelH-10-0.8-0.2& value& 1590.97& 1416.80& 3415.92& 5415.13& 4415.43& 5422.47& \bf 1416.10& 1416.24\\\nopagebreak
& bound& -Inf& -Inf& 1415.68& 1415.62& 1415.70& 1404.47& \bf 1416.10& -Inf\\\nopagebreak
& runtime& 0.05& 22.66& 1.20& 12.03& 3797.40& 91.16& 0.13& 85.46\\\hline
modelH-2-0.8-0.2& value& 1603.85& 1423.42& 4423.49& 6422.84& 3423.03& 5436.16& \bf 1422.89& \bf 1422.89\\\nopagebreak
& bound& -Inf& -Inf& 1422.79& 1422.78& 1422.79& 1411.56& \bf 1422.89& -Inf\\\nopagebreak
& runtime& 0.05& 21.34& 10.00& 6.67& 4051.20& 101.83& 0.11& 113.24\\\hline
modelH-3-0.8-0.2& value& 1596.11& \bf 1381.14& \bf 1381.14& \bf 1381.14& \bf 1381.14& 4389.78& \bf 1381.14& 1381.19\\\nopagebreak
& bound& -Inf& -Inf& \bf 1381.14& 1381.14& 1381.14& 1371.29& \bf 1381.14& -Inf\\\nopagebreak
& runtime& 0.06& 8.02& 4.50& 7.79& 8.84& 112.52& 0.11& 63.51\\\hline
modelH-4-0.8-0.2& value& 1595.12& 1427.56& 5427.63& 5426.48& 3427.27& 2432.97& \bf 1427.17& \bf 1427.17\\\nopagebreak
& bound& -Inf& -Inf& 1426.58& 1426.56& 1426.58& 1416.80& \bf 1427.17& -Inf\\\nopagebreak
& runtime& 0.04& 21.18& 9.40& 8.29& 3892.65& 116.38& 0.13& 125.01\\\hline
modelH-5-0.8-0.2& value& 1566.58& 3383.89& 6383.61& 4383.52& 6382.77& 4390.47& \bf 1383.69& 1383.77\\\nopagebreak
& bound& -Inf& -Inf& 1383.25& 1383.23& 1383.30& 1371.94& \bf 1383.69& -Inf\\\nopagebreak
& runtime& 0.04& 21.05& 8.45& 5.44& 3902.54& 112.86& 0.18& 99.08\\\hline
modelH-6-0.8-0.2& value& 1588.33& 2402.30& 2402.17& 2402.60& 5401.70& 3406.27& \bf 1402.34& 1402.60\\\nopagebreak
& bound& -Inf& -Inf& 1402.01& 1401.77& 1402.01& 1393.05& \bf 1402.34& -Inf\\\nopagebreak
& runtime& 0.03& 20.80& 2.69& 22.61& 3778.21& 101.74& 0.11& 126.40\\\hline
modelH-7-0.8-0.2& value& 1583.36& 1403.61& 3403.70& 5402.97& 5403.24& 6418.08& \bf 1403.25& 1403.69\\\nopagebreak
& bound& -Inf& -Inf& 1403.08& 1403.07& 1403.08& 1391.87& \bf 1403.25& -Inf\\\nopagebreak
& runtime& 0.04& 20.80& 2.50& 11.98& 4124.95& 103.95& 0.15& 94.36\\\hline
modelH-8-0.8-0.2& value& 1574.64& 3368.65& 3368.65& 3368.66& 1368.55& \bf 1368.33& \bf 1368.33& \bf 1368.33\\\nopagebreak
& bound& -Inf& -Inf& 1368.29& 1368.29& 1368.33& 1368.23& \bf 1368.33& -Inf\\\nopagebreak
& runtime& 0.05& 20.66& 11.21& 5.09& 3740.80& 92.39& 0.15& 86.69\\\hline
modelH-9-0.8-0.2& value& 1577.25& 1385.00& 1385.23& 2385.04& 3385.06& \bf 1384.86& \bf 1384.86& 1384.95\\\nopagebreak
& bound& -Inf& -Inf& 1384.82& 1384.82& 1384.82& 1384.81& \bf 1384.86& -Inf\\\nopagebreak
& runtime& 0.03& 3.61& 3.15& 4.75& 3824.62& 82.98& 0.11& 73.29\\\hline
\hline
\textbf{mean energy} & & 1587.13& 1441.43& 1694.72& 3300.67& 2800.54& 4007.73& 1400.57& 1400.68\\\hline
\textbf{mean bound} & & -Inf& -Inf& 1400.33& 1400.30& 1400.35& 1392.01& 1400.57& -Inf\\\hline
\textbf{mean runtime} & & 0.05& 18.20& 6.36& 9.68& 3496.50& 101.56& 0.13& 97.37\\\hline
\textbf{best value} & & 0.00& 10.00& 10.00& 10.00& 10.00& 20.00& 100.00& 40.00\\\hline
\textbf{best bound} & & 0.00& 0.00& 10.00& 0.00& 0.00& 0.00& 100.00& 0.00\\\hline
\textbf{verified opt} & & 0.00& 0.00& 10.00& 0.00& 0.00& 0.00& 100.00& 0.00\\\hline
\end{longtable}}

{\scriptsize\begin{longtable}{|l|l|r|r|r|} 
\caption[]{inclusion}
\\\hline
\textbf{inclusion}& & BCD& FW& PGD\\\hline
\endfirsthead
\hline
\textbf{inclusion}& & BCD& FW& PGD\\\hline
\endhead
modelH-1-0.8-0.2& value& 12435.37& 7419.38& 7421.24\\\nopagebreak
& bound& -Inf& -Inf& -Inf\\\nopagebreak
& runtime& 0.14& 44.22& 67.47\\\hline
modelH-10-0.8-0.2& value& 15446.57& 7427.81& 5424.26\\\nopagebreak
& bound& -Inf& -Inf& -Inf\\\nopagebreak
& runtime& 0.14& 2.76& 16.90\\\hline
modelH-2-0.8-0.2& value& 10430.00& 5425.92& 5425.74\\\nopagebreak
& bound& -Inf& -Inf& -Inf\\\nopagebreak
& runtime& 0.14& 11.55& 57.53\\\hline
modelH-3-0.8-0.2& value& 15397.00& 1382.80& 1382.23\\\nopagebreak
& bound& -Inf& -Inf& -Inf\\\nopagebreak
& runtime& 0.14& 20.57& 19.35\\\hline
modelH-4-0.8-0.2& value& 15447.30& 4427.73& 4427.66\\\nopagebreak
& bound& -Inf& -Inf& -Inf\\\nopagebreak
& runtime& 0.13& 8.25& 109.47\\\hline
modelH-5-0.8-0.2& value& 9391.02& 6385.98& 6385.44\\\nopagebreak
& bound& -Inf& -Inf& -Inf\\\nopagebreak
& runtime& 0.13& 6.26& 32.41\\\hline
modelH-6-0.8-0.2& value& 13420.27& 5407.69& 3403.83\\\nopagebreak
& bound& -Inf& -Inf& -Inf\\\nopagebreak
& runtime& 0.14& 36.05& 24.21\\\hline
modelH-7-0.8-0.2& value& 11438.71& 10411.17& 11498.09\\\nopagebreak
& bound& -Inf& -Inf& -Inf\\\nopagebreak
& runtime& 0.13& 18.45& 72.97\\\hline
modelH-8-0.8-0.2& value& 14385.72& 6376.91& 6375.75\\\nopagebreak
& bound& -Inf& -Inf& -Inf\\\nopagebreak
& runtime& 0.14& 35.24& 80.66\\\hline
modelH-9-0.8-0.2& value& 7393.92& 3386.31& 3385.93\\\nopagebreak
& bound& -Inf& -Inf& -Inf\\\nopagebreak
& runtime& 0.14& 28.90& 29.45\\\hline
\hline
\textbf{mean energy} & & 12518.59& 5805.17& 5513.02\\\hline
\textbf{mean bound} & & -Inf& -Inf& -Inf\\\hline
\textbf{mean runtime} & & 0.14& 21.23& 51.04\\\hline
\textbf{best value} & & 0.00& 0.00& 0.00\\\hline
\textbf{best bound} & & 0.00& 0.00& 0.00\\\hline
\textbf{verified opt} & & 0.00& 0.00& 0.00\\\hline
\end{longtable}}

{\scriptsize\begin{longtable}{|l|l|r|r|r|r|r|r|r|r|} 
\caption[]{stereo}
\\\hline
\textbf{stereo}& & $\alpha$-Fusion& TRBP& ADDD& MPLP& MPLP-C& BUNDLE& SRMP& ADMM\\\hline
\endfirsthead
\hline
\textbf{stereo}& & $\alpha$-Fusion& TRBP& ADDD& MPLP& MPLP-C& BUNDLE& SRMP& ADMM\\\hline
\endhead
art\_small& value& 13262.49& 13336.35& 13543.70& NaN& NaN& 15105.28& \bf 13091.20& 13297.79\\\nopagebreak
& bound& -Inf& -Inf& 12925.76& NaN& NaN& 12178.62& \bf 13069.30& -Inf\\\nopagebreak
& runtime& 50.99& 3744.91& 3096.10& NaN& NaN& 3845.89& 3603.89& 3710.92\\\hline
cones\_small& value& 18582.85& 18640.25& 18763.13& NaN& NaN& 20055.65& \bf 18433.01& 18590.87\\\nopagebreak
& bound& -Inf& -Inf& 18334.00& NaN& NaN& 17724.56& \bf 18414.29& -Inf\\\nopagebreak
& runtime& 48.89& 3660.77& 7506.15& NaN& NaN& 3814.74& 3603.11& 3659.15\\\hline
teddy\_small& value& 14653.53& 14680.21& 14804.46& NaN& NaN& 15733.15& \bf 14528.74& 14715.83\\\nopagebreak
& bound& -Inf& -Inf& 14374.12& NaN& NaN& 13981.71& \bf 14518.03& -Inf\\\nopagebreak
& runtime& 50.99& 3670.35& 3535.79& NaN& NaN& 3820.05& 3603.49& 3620.84\\\hline
venus\_small& value& 9644.78& 9692.80& 9796.44& NaN& NaN& 9990.68& \bf 9606.34& 9669.62\\\nopagebreak
& bound& -Inf& -Inf& 9377.05& NaN& NaN& 9402.97& \bf 9601.86& -Inf\\\nopagebreak
& runtime& 49.24& 3627.58& 3761.29& NaN& NaN& 3774.66& 3603.14& 3657.60\\\hline
\hline
\textbf{mean energy} & & 14035.91& 14087.40& 14226.93& NaN& NaN& 15221.19& 13914.82& 14068.53\\\hline
\textbf{mean bound} & & -Inf& -Inf& 13752.73& NaN& NaN& 13321.96& 13900.87& -Inf\\\hline
\textbf{mean runtime} & & 50.03& 3675.90& 4474.83& NaN& NaN& 3813.84& 3603.41& 3662.13\\\hline
\textbf{best value} & & 0.00& 0.00& 0.00& 0.00& 0.00& 0.00& 100.00& 0.00\\\hline
\textbf{best bound} & & 0.00& 0.00& 0.00& 0.00& 0.00& 0.00& 100.00& 0.00\\\hline
\textbf{verified opt} & & 0.00& 0.00& 0.00& 0.00& 0.00& 0.00& 0.00& 0.00\\\hline
\end{longtable}}

{\scriptsize\begin{longtable}{|l|l|r|r|r|} 
\caption[]{stereo}
\\\hline
\textbf{stereo}& & BCD& FW& PGD\\\hline
\endfirsthead
\hline
\textbf{stereo}& & BCD& FW& PGD\\\hline
\endhead
art\_small& value& 13896.67& 13696.50& 13929.06\\\nopagebreak
& bound& -Inf& -Inf& -Inf\\\nopagebreak
& runtime& 60.63& 1407.07& 3648.00\\\hline
cones\_small& value& 18926.70& 18776.26& 19060.17\\\nopagebreak
& bound& -Inf& -Inf& -Inf\\\nopagebreak
& runtime& 57.40& 2111.63& 3669.24\\\hline
teddy\_small& value& 14998.31& 14891.12& 15193.23\\\nopagebreak
& bound& -Inf& -Inf& -Inf\\\nopagebreak
& runtime& 60.08& 1626.66& 3671.59\\\hline
venus\_small& value& 9767.21& 9726.27& 9992.13\\\nopagebreak
& bound& -Inf& -Inf& -Inf\\\nopagebreak
& runtime& 60.27& 1851.40& 3670.82\\\hline
\hline
\textbf{mean energy} & & 14397.22& 14272.54& 14543.65\\\hline
\textbf{mean bound} & & -Inf& -Inf& -Inf\\\hline
\textbf{mean runtime} & & 59.59& 1749.19& 3664.92\\\hline
\textbf{best value} & & 0.00& 0.00& 0.00\\\hline
\textbf{best bound} & & 0.00& 0.00& 0.00\\\hline
\textbf{verified opt} & & 0.00& 0.00& 0.00\\\hline
\end{longtable}}
\end{document}